\documentclass[preprint,12pt,3p]{elsarticle}

\usepackage{color}

\usepackage{amstext}

\usepackage{fancyhdr,lipsum}

\usepackage{microtype}

\usepackage{varioref}
\usepackage{hyperref}
\usepackage{cleveref}
\crefname{subsection}{subsection}{subsections}

\usepackage{indentfirst}
\usepackage{fancyhdr}
\usepackage{graphicx}

\usepackage{amssymb}
\usepackage{amsmath}
\usepackage{latexsym}
\usepackage{amsthm}
\usepackage{eucal}
\usepackage{faktor}

\usepackage[title]{appendix}

\newtheorem{theorem}{Theorem}[section]
\newtheorem*{theorem*}{Theorem}

\newtheorem{proposition}[theorem]{Proposition}
\newtheorem*{proposition*}{Proposition}

\newtheorem*{corollary*}{Corollary}
\theoremstyle{definition}
\newtheorem{definition}[theorem]{Definition}
\newtheorem{example}[theorem]{Example}
\newtheorem*{example*}{Example}
\theoremstyle{remark}
\newtheorem{remark}[theorem]{Remark}
\newtheorem*{remark*}{Remark}

\newcommand{\R}{\mathbb{R}}
\newcommand{\Homeo}{\mathrm{Homeo}}
\newcommand{\GENEO}{\mathrm{GENEO}}
\newcommand{\GENEOH}{\mathrm{GENEO,H}}

\begin{document}

\begin{frontmatter}

\title{
Towards a topological-geometrical theory of group equivariant non-expansive operators for data analysis and machine learning}

\author[label4]{Mattia G. Bergomi}
\address[label4]{Champalimaud Research, Champalimaud Center for the Unknown - Lisbon, Portugal}
\ead{mattia.bergomi@neuro.fchampalimaud.org}

\author[label1,label2]{Patrizio Frosini\corref{cor1}}
\address[label1]{Department of Mathematics, University of Bologna}
\address[label2]{Advanced Research Center on Electronic System ``Ercole De Castro'', University of Bologna\fnref{label4}}

\cortext[cor1]{Corresponding author}

\ead{patrizio.frosini@unibo.it}

\author[label3]{Daniela Giorgi}
\address[label3]{Italian National Research Council, Institute of Information Science and Technologies ``Alessandro Faedo''}
\ead{daniela.giorgi@isti.cnr.it}

\author[label1,label2]{Nicola Quercioli}
\ead{nicola.quercioli2@unibo.it}

\begin{abstract}
The aim of this paper is to provide a general mathematical framework for group equivariance in the machine learning context. The framework builds on a synergy between persistent homology and the theory of group actions. We define group equivariant non-expansive operators (GENEOs), which are maps between function spaces associated with groups of transformations. We study the topological and metric properties of the space of GENEOs to evaluate their approximating power and set the basis for general strategies to initialise and compose operators. We begin by defining suitable pseudo-metrics for the function spaces, the equivariance groups, and the set of non-expansive operators. Basing on these pseudo-metrics, we prove that the space of GENEOs is compact and convex, under the assumption that the function spaces are compact and convex. These results provide fundamental guarantees in a machine learning perspective. We show examples on the MNIST and fashion-MNIST datasets. By considering isometry-equivariant non-expansive operators, we describe a simple strategy to select and sample operators, and show how the selected and sampled operators can be used to perform both classical metric learning and an effective initialisation of the kernels of a convolutional neural network.
\end{abstract}


\end{frontmatter}



\section{Introduction}
\label{sec:introduction}
Deep learning-based algorithms reached human or superhuman performance in many real-world tasks. Beyond the extreme effectiveness of deep learning, one of the main reasons for its success is that raw data are sufficient---if not even more suitable than hand-crafted features---for these algorithms to learn a specific task. However, only few attempts have been made to create formal theories allowing for the creation of a controllable and interpretable framework, in which deep neural networks can be formally defined and studied. Furthermore, if learning directly from raw data allows one to outclass human feature engineering, the architectures of deep networks are growing more and more complex, and often are as task-specific as hand-crafted features used to be.

We aim at providing a general mathematical framework, where any agent capable of acting on a certain dataset (e.g. deep neural networks) can be formally described as a collection of operators acting on the data. To motivate our model, we assume that data cannot be studied directly, but only through the action of agents that measure and transform them.
Consequently, our model stems from a functional viewpoint. By interpreting data as points of a function space, it is possible to learn and optimise operators defined on the data. In other words, we are interested in the space of transformations of the data, rather than the data themselves.

Albeit unformalised, this idea is not new in deep learning. For instance, one of the main features of convolutional neural networks \cite{lecun1995convolutional} is the election of convolution as the operator of choice to act on the data. The convolutional kernels learned by optimising a loss function are operators that map an image to a new one that, for instance, is more easily classifiable. Moreover, convolutions are operators equivariant with respect to translations (at least in the ideal continuous case). We believe that the restriction to a specific family of operators and the equivariance with respect to interpretable transformations are key aspects of the success of this architecture. In our theory, operators are thought of as instruments allowing an agent to provide a measure of the world, as the kernels learned by a convolutional neural network allow a classifier to spot essential features to recognise objects belonging to the same category. Equivariance with respect to the action of a group (or a set) of transformations corresponds to the introduction of symmetries in the function space where data are represented. This allows us to both gain control on the nature of the learned operators, as well as drastically reduce the dimensionality of the space of operators to be explored during learning. Such a goal is in line with the recent interest for invariant representations in machine learning (cf., e.g., \cite{AnRoPo16}).

We make use of topological data analysis to describe spaces of group equivariant non-expansive operators (GENEOs). GENEOs are maps between function spaces associated with groups of transformations. We study the topological and metric properties of the space of GENEOs to evaluate their approximating power and set the basis for general strategies to initialise, compose operators and eventually connect them hierarchically to form operator networks. Our first contribution is to define suitable pseudo-metrics for the function spaces, the equivariance groups, and the set of non-expansive operators. Basing on these pseudo-metrics, we prove that the space of GENEOs is compact and convex, under the assumption that the function spaces are compact and convex. These results provide fundamental and provable guarantees for the goodness of this operator-based approach in a machine learning perspective: Compactness, for instance,  guarantees that any operator belonging to a certain space can be approximated by a finite number of operators sampled in the same space.

Our study of the space of GENEOs takes advantage of recent results in topological data analysis, in particular in the theory of persistent homology~\cite{FrJa16}. Our approach also generalises standard group equivariance to set equivariance, which seems much more suitable for the representation of intelligent agents.

To conclude, we validate our model with examples on the MNIST, fashion-MNIST and CIFAR10 datasets. These applications are aimed at proving the effectiveness on discrete examples, of the metrics defined and the theorems proved in the continuous case. By considering isometry equivariant non-expansive operators (IENEOs), we describe two simple algorithms allowing the selection and sampling of IENEOs based on few labelled samples taken from the dataset. We show how selected and sampled operators can be used to perform both classical metric learning and effective initialisation of the kernels of a convolutional neural network.

Our main contribution is a general framework to previous works on group equivariance in deep learning context~\cite{cohen2016group,worrall2017harmonic}. We believe that the formal foundation of our model is suitable to start a new theory of \textit{deep-learning engineering}, and that novel research lines will stem from the synergy of machine learning and topology.
This synergy is object of study by more and more researchers, focusing both on the treatment of data via TDA before applying classical  machine learning~\cite{AdEmal17,PuXiXi18}, and the analysis of the topology of convolutional neural networks~\cite{GaCa18}. However, our approach differs from the previous ones in that it focuses on a new theoretical setting, based on the introduction of new topologies and metrics.

The paper is structured as follows. In~\Cref{sec:episet} the epistemological foundations of our model are discussed. The mathematical background in topological persistence is provided in~\Cref{sec:mathbkg}. \Cref{sec:mathmod} details the mathematical model for data, transformations, and GENEOs. ~\Cref{sec:comp_conv} proves the compactness and convexity of the space of GENEOs, under suitable hypotheses. New results in persistent homology to define computable metrics in the space of GENEOs and in the space of data are presented in~\Cref{sec:newPH}, along with the extension of the theory from group to set equivariance.  Finally, in~\Cref{sec:results}, we describe two algorithms to select and sample operators in the discrete case, and show examples on the MNIST and fashion-MNIST datasets. A Python package allowing to reproduce the computational experiments described in~\Cref{sec:results} is available in \href{https://gitlab.com/mattia.bergomi/geneos}{gitlab.com/mattia.bergomi/geneos}.

\section{Epistemological setting}
\label{sec:episet}

Our mathematical model is justified by an epistemological background which revolves around the following assumptions:

\begin{enumerate}
\item \label{a} Data are represented as functions defined on topological spaces, since only
data that are stable with respect to a certain criterion (e.g., with respect to some kind of
measurement) can be considered for applications, and stability requires a topological
structure.
\item \label{b} Data cannot be studied in a direct and absolute way. They are only knowable through
acts of transformation made by an agent. From the point of view of data analysis, only
the pair (data, agent) matters. In general terms, agents are not endowed with purposes or goals: they are just ways
and methods to transform data. Acts of measurement are a particular class of acts of transformation.
\item \label{c} Agents are described by the way they transform data while respecting some kind of
invariance. In other words, any agent can be seen as a group equivariant operator acting
on a function space.
\item \label{d} Data similarity depends on the output of the considered agent.
\end{enumerate}

In other words, in our framework we assume that the analysis of data is replaced by the
analysis of the pair (data, agent). Since an agent can be seen as a
group equivariant operator, from the mathematical viewpoint our purpose consists in
presenting a good topological theory of suitable operators of this kind, representing agents.
For more details, we refer the interested reader to \cite{Fr16}.

\section{Mathematical background}
\label{sec:mathbkg}

Our mathematical model builds on functional analysis and Topological Data Analysis (TDA). TDA is an emerging field of research which studies topological approaches to explore and make sense of complex, high-dimensional data, such as artificial and biological networks \cite{Ca2009,Ca2013}. The basic idea is that topology can help to recognize patterns within data, and therefore to turn data into useful knowledge. One of the main concepts in TDA is Persistent Homology (PH), a mathematical tool that captures topological information at multiple scales. Our mathematical model proposes an integration between the theory of group actions and persistent homology.

In summary, persistent homology allows to represent the topological and geometrical features of a topological space $X$ (e.g. an image, a $3$-dimensional mesh) as it is \textit{seen} by a continuous, real-valued function $\varphi$ defined on the space. The homology functor (see for instance~\cite{hatcher2005algebraic}) is used to encode the information of the pair $(X,\varphi)$ in the form of \textit{persistence diagrams}.
In other words, we can associate each continuous function $\varphi:X\to\R$ with a persistence diagram $D_\varphi$, that is represented by a discrete collection of points in the real plane. Beyond the technicalities that are needed to define the concept of persistence diagram, two important points are to be stressed. First, persistence diagrams can be quickly computed. Second, an easily computable distance $\delta_{\mathrm{match}}$ between persistence diagrams is available and gives a lower bound for the max-norm distance between functions: $\delta_{\mathrm{match}}(D_{\varphi_1},D_{\varphi_2})\le\|\varphi_1-\varphi_2\|_\infty$. It follows that the bottleneck distance $\delta_{\mathrm{match}}$ between persistence diagrams can be used as an efficient proxy for the max-norm distance between real-valued functions. Since our approach is deeply rooted in the comparison of real-valued functions, persistence diagrams are a key tool in our model.
The definition of persistence diagram and of the bottleneck distance $\delta_{\mathrm{match}}$ are intuitively depicted in~\cref{fig:persistence_bkg} and rapidly formalised in what follows. We refer the reader to \cite{BiDFFaal08,CaZo09,EdHa08} for further details.

\subsection{Persistent Homology}
\label{sectionPH}

In PH, data are modelled as objects in a metric space. The first step is to filter the data so to obtain a family of nested topological spaces that captures the topological information at multiple scales. A common way to obtain a filtration is by sublevel sets of a continuous function, hence the name \emph{sublevelset persistence}. Let $\varphi$ be a real-valued continuous function on a topological space $X$. Persistent homology represents the changes of the homology groups of the sub-level set $X_t=\varphi^{-1}((-\infty,t])$ varying $t$ in $\mathbb{R}$. We can see the parameter $t$ as an increasing time, whose changes produce the birth and the death of $k$-dimensional holes in the sub-level set $X_t$. We observe that the number of independent 0-dimensional holes of $X_t$ equals the number of connected components of $X_t$ minus one, 1-dimensional holes refer to tunnels and 2-dimensional holes to voids.

\begin{definition}
If $u,v \in \mathbb{R}$ and $u<v$, we can consider the inclusion $i$ of $X_u$ into $X_v$. If $\check{H}$ denotes the \v{C}ech homology functor, such an inclusion induces a homomorphism $i_k: \check{H}_k(X_u) \rightarrow \check{H}_k(X_v)$ between the homology groups of $X_u$ and $X_v$ in degree $k$. The group $PH_k^\varphi(u,v):=i_k(\check{H}_k(X_u))$ is called the $k$th \textit{persistent homology group} with respect to the function $\varphi: X \rightarrow \mathbb{R}$, computed at the point $(u,v)$. The rank $r_k(\varphi)(u,v)$ of $PH_k^\varphi(u,v)$ is said the $k$th \emph{persistent Betti numbers function} (PBN) with respect to the function $\varphi: X \rightarrow \mathbb{R}$, computed at the point $(u,v)$.
\end{definition}




\begin{figure}[tb]
  \centering
  \includegraphics[width=.9\textwidth]{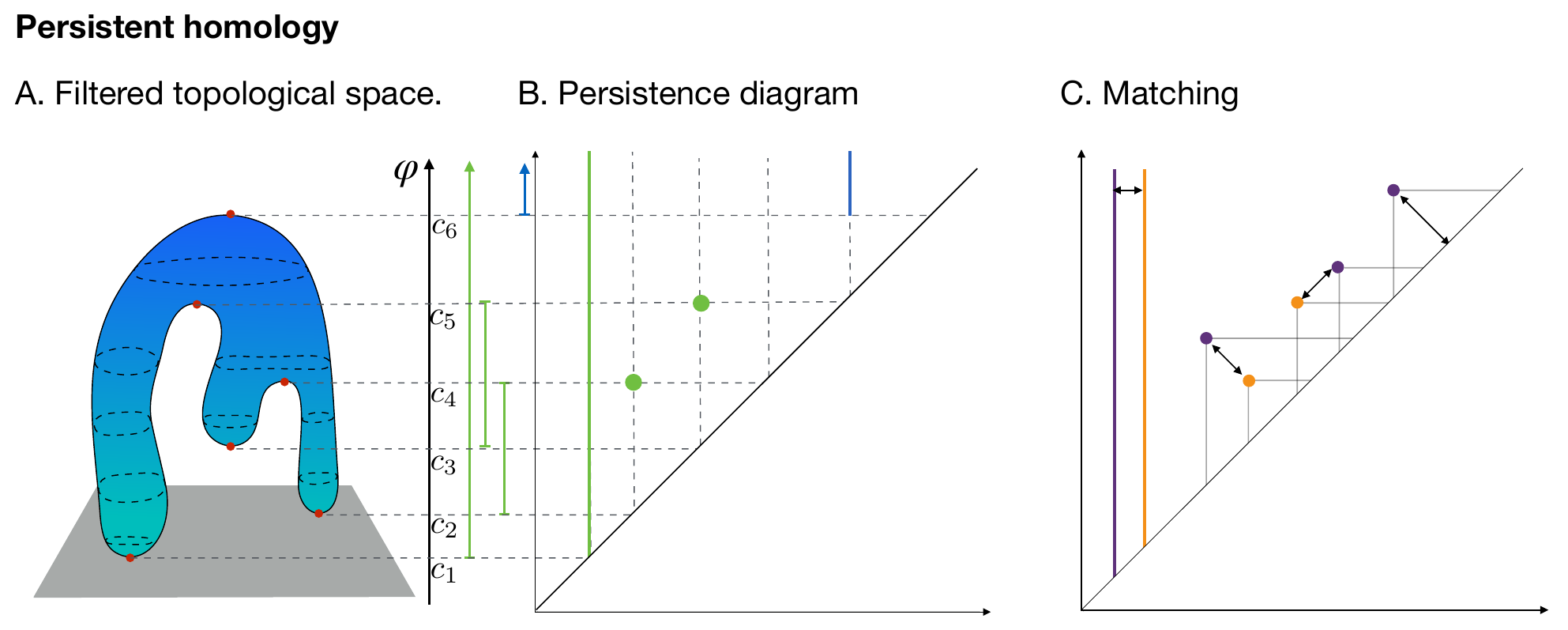}
  \caption{In persistent homology we consider pairs composed by a topological space and a continuous function defined on the topological space of interest. The (homological) critical values of the function induce naturally a sublevel set filtration of the topological space. In panel A, a topological sphere is filtered by considering the critical values of the height function. We obtain a filtration by considering the sequence of nested sublevel sets ordered according to the natural order on the critical values. The evolution throughout the filtration of the number of generators of the $k$th homology groups (i.e. the number of $k$-dimensional holes) is represented as a persistence diagram in Panel B. $0$-dimensional holes, or connected components are represented as green cornerpoints. The void generated when considering the last sublevel set, corresponding to the entire space, generates the cornerline depicted in blue. A distance between two persistence diagrams can be computed as an optimal matching of cornerpoint. The matching process is depicted in Panel C. Note how non-matchable cornerpoints can be associated to their projection on the diagonal. \label{fig:persistence_bkg}}
\end{figure}

Persistent Betti numbers functions can be completely described by multisets called \textit{persistence diagrams}.
The $k$th persistence diagram is the multiset of all the pairs $p_j=(b_j,d_j)$, where $b_j$ and $d_j$ are the times of birth and death of the $j$th $k$-dimensional hole, respectively. When a hole never dies, we set its time to death equal to $\infty$. The multiplicity $m(p_j)$ says how many holes share both the time of birth $b_j$ and the time of death $d_j$. For technical reasons, the points $(t,t)$ on the diagonal are added to each persistence diagram, each one with infinite multiplicity. 

Each persistence diagram $D$ can contain an infinite number of points.
For every $q \in \Delta^{*}:= \{(x,y) \in \mathbb{R}^2 \ : \ x<y \} \cup \{(x,\infty):x\in \mathbb{R} \}$, the  equality $m(q)=0$ means that $q$ does not belong to the persistence diagram $D$.
We define on $\bar{\Delta}^{*}:=\{(x,y) \in \mathbb{R}^2 \ : \ x\le y \} \cup \{(x,\infty):x\in \mathbb{R} \}$ a pseudo-metric as follows
\begin{equation}d^{*}((x,y),(x',y')):= \min \left\{\max \{|x-x'|,|y-y'|\},\max  \left \{\frac{y-x}{2},\frac{y'-x'}{2}\right \}\right  \}
\end{equation}
by agreeing that $\infty - y=\infty, \ y-\infty= -\infty$ for $y\ne \infty, \ \infty - \infty=0, \ \frac{\infty}{2}=\infty, \ |\pm \infty|=\infty, \ \min\{\infty,c\}=c, \ \max \{\infty, c\}=\infty$.

The pseudo-metric $d^{*}$ between two points $p$ and $p'$ takes the smaller value between the cost of moving $p$ to $p'$ and the cost of moving $p'$ and $p$ onto $\Delta:=\{(x,y) \in \mathbb{R}^2 \ : \ x=y \}$. Obviously, $d^{*}(p,p')=0$ for every $p, p' \in \Delta$. If $p\in \Delta^{+}:= \{(x,y) \in \mathbb{R}^2 \ : \ x<y \}$ and $p'\in \Delta$, then $d^{*}(p,p')$ equals the distance, induced by the max-norm, between $p$ and $\Delta$. Points at infinity have a finite distance only to the other points at infinity, and their distance equals the Euclidean distance between abscissas.

We can compare persistence diagrams by means of the \emph{bottleneck distance} (also called \emph{matching distance}) $\delta_{\mathrm{match}}$.

\begin{definition} \label{def:matching}
Let $D, \ D'$ be two persistence diagrams. We define the \emph{bottleneck distance} $\delta_{\mathrm{match}}$ between $D$ and $D'$ by setting
\begin{equation}\delta_{\mathrm{match}}\left(D,D'\right):=\inf_{\sigma}\sup_{p\in D}d^{*}\left(p,\sigma\left(p\right)\right),
\end{equation}
where $\sigma$ varies in the set of all bijections from the multiset $D$ to the multiset $D'$.
\end{definition}

For further informations about persistence diagrams and the bottleneck distance, we refer the reader to \cite{EdHa08,CSEdHa07}. Each persistent Betti numbers function is associated with exactly one persistence diagram, and (if we use \v{C}ech homology) every persistence diagram is associated with exactly one persistent Betti numbers function. Then the metric $\delta_{\mathrm{match}}$ induces a pseudo-metric $d_{\mathrm{match}}$ on the sets of the persistent Betti numbers functions \cite{CeDFFeal13}.





\section{Mathematical model}
\label{sec:mathmod}

In our mathematical model, data are represented as function spaces, that is, as sets of real-valued functions on some topological space (\Cref{sec:datarepr}). Function spaces come with invariance groups representing the transformations on data which are admissible for some agent (\Cref{sec:invgroup}). The groups of transformations are specific to different agents, and can be either learned or part of prior knowledge. The operators on data are then defined as group equivariant non-expansive operators (GENEOs) (\Cref{sec:operators}).

\subsection{Data representation}
\label{sec:datarepr}

Let us consider a set $X \neq \emptyset $ and a topological subspace $\varPhi$ of the set of all bounded
functions $\varphi$ from $X$ to $\mathbb{R}$, denoted by $\mathbb{R}^X_{b}$ and endowed with the topology induced by the distance
\begin{equation} \label{eq:diphi}
D_\varPhi(\varphi_1,\varphi_2):=\left \| \varphi_1-\varphi_2\right \|_{\infty}.
\end{equation}

If $\varPhi$ is compact, then it is also bounded, i.e., there exists  a  non-negative real  value $L$, such that $\|\varphi \|_{\infty}\le L$ for every $\varphi \in \varPhi$.
We can think of $X$ as the space where one makes measurements, and of $\varPhi$ as the set of admissible measurements, called \emph{set of admissible functions}. In other words, $\varPhi$ is the set of functions from $X$ to $\mathbb{R}$ that can be produced by measuring instruments. For example, an image can be represented as a function $\varphi$ from the real plane $X$ to the real numbers.


To quantify the distance between two points $x_1,x_2\in X$, we compare the values taken at $x_1$ and $x_2$ by the admissible functions in $\varPhi$. Therefore, we endow $X$ with the extended pseudo-metric\footnote{We recall that a pseudo-metric is just a distance $d$ without the property that $d(a,b)=0$ implies $a=b$. An extended pseudo-metric is a pseudo-metric that may take the value $\infty$. If $\varPhi$ is bounded, then $D_X$ is a pseudo-metric.} $D_X$ defined by setting
\begin{equation}
D_X(x_1,x_2)=\sup_{\varphi \in \varPhi}|\varphi (x_1) - \varphi (x_2) |
\end{equation}
for every $x_1,x_2\in X$ (see Appendix~\ref{appendix-D_X}).

The assumption behind the definition of $D_X$ is that two points can be distinguished only if they assume different values for some admissible function. As an example, if $\varPhi$ contains only constant functions, no discrimination can be made between points in $X$ and hence $D_X(x_1,x_2)$ vanishes for every $x_1,x_2\in X$.


The pseudo-metric space $(X,D_X)$ can be considered as a topological space by choosing as a base $\mathcal{B}_{D_X}$ the collection of all the sets
\begin{equation}
B_{X}(x,\varepsilon)=\{x' \in X : D_X(x,x') < \varepsilon \}
\end{equation}
where $\varepsilon > 0$ and $x \in X$ (see \cite{Ga64}).

The reason to endow the measurement space $X$ with a topology, rather than considering just a set, follows from the need of formalizing the assumption that data are stable. 
To formalize stability we have to use a topology (or a pseudo-metric inducing a topology).


It is interesting to stress the link between the topology $\tau_{D_X}$ associated with $D_X$ and the initial topology\footnote{We recall that $\tau_{\mathrm{in}}$ is the coarsest topology on $X$ such that each function $\varphi\in\varPhi$ is continuous.
Explicitly, the open sets in $\tau_{\mathrm{in}}$ are the sets that can be obtained as unions of finite intersections of sets $\varphi^{-1}(U)$, where $\varphi\in\varPhi$ and $U\in \mathcal{T}_E$.
In other words, a base $\mathcal{B}_{\mathrm{in}}$ of $\tau_{\mathrm{in}}$ is given by the collection of all sets that can be represented as $\bigcap_{i \in I}\varphi^{-1}_i( U_i )$, where $I$ is a finite set of indexes and $\varphi_i\in\varPhi$, $U_i\in \mathcal{T}_E$ for every $i\in I$ \cite{Ga64}.
} $\tau_{\mathrm{in}}$ on $X$ with respect to $\varPhi$, when we take the Euclidean topology $\mathcal{T}_E$ on $\mathbb{R}$.
\begin{theorem}
\label{tau=tau'}
The topology  $\tau_{D_X}$ on $X$ induced by the pseudo-metric $D_X$ is finer than the initial topology $\tau_{\mathrm{in}}$ on  $X$ with respect to $\varPhi$.
If $\varPhi$ is totally bounded, then the topology  $\tau_{D_X}$ coincides with $\tau_{\mathrm{in}}$.
\end{theorem}
(The proof is in Appendix~\ref{AppendixProofs}.)\\



Since $\tau_{\mathrm{in}}$ is the coarsest topology on $X$ such that $\varphi \in \varPhi$ is continuous, Theorem \ref{tau=tau'} guarantees that the assumption that the functions are continuous is not restrictive in practice, for example while dealing with images, which often contain discontinuities. Indeed, our functions are not required to be continuous with respect to other topologies (e.g., the Euclidean topology $\tau_E$ on $X=\R^2$). \\ 

In general $X$ is not compact with respect to the topology $\tau_{D_X}$, even if $\varPhi$ is compact. For example, if $X$ is the open interval $]0,1[$ and $\varPhi$ contains only the identity from $]0,1[$ to $]0,1[$, the topology induced by $D_X$ is simply the Euclidean topology and hence $X$ is not compact. However, the next result holds.

\begin{theorem}
\label{Xcomplete}
If $\varPhi$ is compact and $X$ is complete then $X$ is also compact.
\end{theorem}
(The proof is in Appendix~\ref{AppendixProofs}.)\\

\subsubsection{A remark on the use of pseudo-metrics}
\label{psm}

The reader could think better to change the pseudo-metric $D_X$ into a metric $D'$ by quotienting out $X$ by the equivalence relation $x_1Rx_2\iff D_X(x_1,x_2)=0$ and defining $D'([x_1],[x_2])=D_X(x_1,x_2)$ for any $[x_1],[x_2]\in X/R$. The reason we do not do this is that several different sets of admissible measurements can be considered on the same set $X$. For two different sets $\varPhi_1$, $\varPhi_2$ of admissible functions, we obtain two different quotient spaces $X/R_1$, $X/R_2$. If we forget about the original space $X$, we lose the possibility of linking the equivalence classes in $X/R_1$ with the ones in $X/R_2$. On the contrary, we prefer to preserve the identity of points in $X$, studying how they link to each other when we change the set $\varPhi$. This observation leads us to work with pseudo-metrics instead of metrics.

Before proceeding, we observe that the map $\pi$ taking each point $x\in X$ to the equivalence class $[x]\in X/R$ is continuous with respect to $D_X$ and $D'$, and surjective. Moreover, $\pi$ takes each ball with respect to $D_X$ to a ball with respect to $D'$,
while the inverse image under $\pi$ of each ball with respect to $D'$ is a ball with respect to $D_X$. It follows that if a subset $S\subseteq X$ is compact (sequentially compact) for $D_X$ then $\pi(S)$ is compact (sequentially compact) for $D'$, and that if a subset $\mathcal{S}\subseteq X/R$ is compact (sequentially compact) for $D'$ then $\pi^{-1}(\mathcal{S})$ is compact (sequentially compact) for $D_X$. Finally, given a sequence $(x_i)$ in $X$, we observe that $(x_i)$ converges to $\bar x$ in $X$ if and only if the sequence
 $([x_i])$ converges to $[\bar x]$ in $X/R$. These facts imply that the development of our theory in terms of pseudo-metrics is not far from the analysis in terms of metrics.

\subsection{Transformations on data}
\label{sec:invgroup}

In our model, we assume that data are transformed through maps from $X$ to $X$ which are $\varPhi$-preserving homeomorphisms with respect to the pseudo-metric $D_X$. Let $\Homeo(X)$ denote the set of homeomorphisms from $X$ to $X$ with respect to $D_X$, and $\Homeo_\varPhi(X)$ denote the set of $\varPhi$-preserving homeomorphisms, namely the homeomorphisms $g\in\Homeo(X)$ such that $\varphi\circ g\in\varPhi$ and $\varphi\circ g^{-1}\in\varPhi$ for every $\varphi\in\varPhi$.

The following Proposition~\ref{propgisometry} implies that $\Homeo_\varPhi(X)$ is exactly the set of all bijections $g:X\to X$ such that $\varphi\circ g\in\varPhi$ and $\varphi\circ g^{-1}\in\varPhi$ for every $\varphi\in\varPhi$.


\begin{proposition}
\label{propgisometry}
If $g$ is a bijection from $X$ to $X$ such that $\varphi\circ g\in\varPhi$ and $\varphi\circ g^{-1}\in\varPhi$ for every $\varphi\in\varPhi$, then $g$ is an isometry\footnote{The definition of isometry between pseudo-metric spaces can be considered as a special case of isometry between metric spaces. Let $(X_1, d_1)$ and $(X_2, d_2)$ be two pseudo-metric spaces. It is easy to check that if $f: X_1\longrightarrow X_2$ is a function verifying the equality $d_1(x,y)=d_2(f(x),f(y))$ for every $x,y\in X_1$, then $f$ is continuous with respect to the topologies induced by $d_1$ and $d_2$. If $f$ verifies the previous equality and is bijective, we say that it is an \emph{isometry} between the considered pseudo-metric spaces.
If $f$ is an isometry, we can trivially observe that $f^{-1}$ is also an isometry, and that $f$ is a homeomorphism.}
 (and hence a homeomorphism) with respect to $D_X$.
\end{proposition}
(The proof is in Appendix~\ref{AppendixProofs}.)\\




\begin{remark}
\label{Homeo<>HomeoPhi}
In general, $\Homeo(X)\neq\Homeo_\varPhi(X)$. As an example, take $X=[0,1]$ and $\varPhi=\{\mathrm{Id}\}$. In this case $D_X(x_1,x_2)=|x_1-x_2|$ and $\Homeo_\varPhi=\{\mathrm{Id}\}$, while $\Homeo$ is the set of all homeomorphisms from the interval $[0,1]$ to itself with respect to the Euclidean distance.
\end{remark}

\begin{remark}
\label{remRgisometry}
For each $g\in \Homeo_\varPhi(X)$, we consider the bijective map $R_g:\varPhi \longrightarrow \varPhi$ defined by setting $R_g(\varphi)= \varphi \circ g$ for every $\varphi\in\varPhi$. We claim that $R_g$ preserves the pseudo-distance $D_\varPhi$ defined by Equality (\ref{eq:diphi}). Indeed, if $\varphi, \varphi' \in \varPhi$ and $g\in G$ then
\begin{align}
D_\varPhi(\varphi\circ g, \varphi' \circ g) &= \sup_{x\in X}|(\varphi \circ g)(x)-(\varphi'\circ g)(x)| \nonumber \\
&= \sup_{x\in X}|\varphi(g(x))-\varphi'(g(x))|\\
&= \sup_{y\in X}|\varphi(y)-\varphi'(y)|= D_\varPhi(\varphi, \varphi'), \nonumber
\end{align}
because $g$ is a bijection. Since $R_g$ is a bijection preserving $D_\varPhi$, then $R_g$ is an isometry with respect to $D_\varPhi$.
\end{remark}

In the rest of this paper we will assume that $\varPhi$ is compact with respect to the topology induced by $D_\varPhi$, and that $X$ is complete (and hence compact) with respect to the topology induced by $D_X$. \\


Let us now consider a subgroup $G$ of the group $\Homeo_\varPhi(X)$. $G$ represents the set of transformations on data for which we require equivariance to be respected.

We can define the pseudo-distance $D_G$ on $G$:
\begin{equation} \label{eq:digi}
D_G(g_1,g_2):=\sup_{\varphi \in \varPhi}D_\varPhi(\varphi \circ g_1,\varphi \circ g_2)
\end{equation}
from $G\times G$ to $\mathbb{R}$ (see Appendix~\ref{appendix-D_X}).

The rationale in the definition of $D_G$ is that in our model every comparison must be based on the max-norm distance between admissible acts of measurement. As a consequence, we define the distance between two homeomorphisms by the difference of their actions on the set $\varPhi$ of possible measurements.

\begin{remark}
$D_G$ can be expressed as:
\begin{equation}
    D_G(g_1,g_2) = \sup_{x \in X} D_X(g_1(x),g_2(x)) = \sup_{x \in X} \sup_{\varphi \in \varPhi} | \varphi(g_1(x)) - \varphi(g_2(x)) |.
\end{equation}
\end{remark}

We can now state the following theorems:

\begin{theorem}
\label{thmRgcontinuous}
$G$ is a topological group with respect to the pseudo-metric topology and the action of $G$ on $\varPhi$ through right composition is continuous.
\end{theorem}
(The proof is in Appendix~\ref{AppendixProofs}.)\\

\begin{theorem}
\label{thmGcompact}
If $G$ is complete then it is also compact with respect to $D_G$.
\end{theorem}
(The proof is in Appendix~\ref{AppendixProofs}.)\\

From now on we will suppose that $G$ is complete (and hence compact) with respect to the topology induced by $D_G$.

\subsubsection{The natural pseudo-distance $d_G$}
\label{d_G}

We define the natural pseudo-distance $d_G$ on the space $\varPhi$ \cite{FrJa16}. The natural pseudo-distance $d_G$ represents the ground truth in our model. It is based on comparing functions, and vanishes for pairs of functions that are equivalent with respect to the action of our group of homeomorphisms $G$, which expresses the equivalences between data.

\begin{definition}
The pseudo-distance $d_G: \varPhi \times \varPhi \rightarrow \mathbb{R}$ is defined by setting
\begin{equation}
d_G(\varphi_1, \varphi_2)=\inf_{g \in G} D_\varPhi(\varphi_1,\varphi_2\circ g).
\end{equation}
It is called the \textit{natural pseudo-distance} associated with the group $G$ acting on $\varPhi$.
\end{definition}


If $G=\{\mathrm{Id}: x \mapsto x \}$, then $d_G$ equals the sup-norm distance $D_\varPhi$ on $\varPhi$.
If $G_1$ and $G_2$ are subgroups of $\Homeo_\varPhi(X)$ and $G_1 \subseteq G_2 $, then the definition of $d_G$ implies that 
\begin{equation}
d_{\Homeo_\varPhi(X)}(\varphi_1, \varphi_2) \le d_{G_2}(\varphi_1, \varphi_2)\le d_{G_1}(\varphi_1, \varphi_2) \le D_\varPhi(\varphi_1, \varphi_2)
\end{equation}
for every $\varphi_1, \ \varphi_2 \in \varPhi$. \\

Though $d_G$ represents the ground truth for data similarity in our model, unfortunately it is difficult to compute. This is also a consequence of the fact that we can easily find subgroups $G$ of $Homeo(X)$ that cannot be approximated with arbitrary precision by smaller finite subgroups of $G$ (e.g., when $G$ is the group of rigid motions of $X = \R^3$).

In the following sections, we show how $d_G$ can be approximated with arbitrary precision by means of a dual approach based on group equivariant non-expansive operators (GENEOs) and persistent homology.

\subsubsection{A remark on the use of homeomorphisms}
\label{remd_G}


The reader could criticize the choice of grounding our approach on the concept of homeomorphism. After all, most of the objects that are considered for purposes of shape comparison ``are not homeomorphic''. Therefore, the definition of natural pseudo-distance could seem not to be sufficiently flexible, since it does not allow to compare non-homeomorphic objects. Though, it is important to note that the space $X$ we use in our model does \emph{not} represent the objects, but the space where one takes measurements \emph{about} the objects. As such, $X$ is unique. For example, two images are considered as functions from the real plane $X$ to the real numbers, independently of the topological properties of the 3D objects represented in the images. If we make two CAT scans, the topological space $X$ is always given by an helix turning many times around a body, and no requirement is made about the topology of such a body. In other words, the topological space $X$ is determined only by the measuring instrument and not by the single object instances. 

\subsection{Group Equivariant Non-Expansive Operators}
\label{sec:operators}

Under the assumptions made in the previous sections, the pair $(\varPhi,G)$ is called a \emph{perception pair}.

Let us now assume that two perception pairs $(\varPhi,G)$, $(\varPsi,H)$ are given together with a fixed homomorphism $T:G\to H$. Each function $F:\varPhi\to\varPsi$ such that $F(\varphi\circ g)=F(\varphi)\circ T(g)$ for every $\varphi\in\varPhi,g\in G$ is said to be a \emph{perception map} from $(\varPhi,G)$ to $(\varPsi,H)$ associated with the homomorphism $T$. More briefly, we will also say that $F$ is a \emph{group equivariant operator}. If $T$ is equal to the identity homomorphism $I: G\longrightarrow G$, we can say that $F$ is a $G$-map. We observe that the functions in $\varPhi$ and the functions in $\varPsi$ are defined on spaces that are generally different from each other.

\begin{remark}
\label{remark:ppasacategory}
Each perception pair $(\varPhi,G)$ can be seen as a category, whose objects are the functions in $\varPhi$ and the morphisms between two functions $\varphi_1,\varphi_2\in\varPhi$ are the elements $g\in G$ such that $\varphi_2=\varphi_1\circ g$. As usual, if $\varphi_2=\varphi_1\circ g$ and $\varphi'_2=\varphi'_1\circ g$ we wish to distinguish $g$ as a morphism between $\varphi_1$ and $\varphi_2$ from $g$ as a morphism between $\varphi'_1$ and $\varphi'_2$, so we make different copies $g_{(\varphi_1,\varphi_2)}$, $g_{(\varphi'_1,\varphi'_2)}$ of the homeomorphism $g$ by labelling it. As natural, $g'_{(\varphi_2,\varphi_3)}\circ g_{(\varphi_1,\varphi_2)}=(g\circ g')_{(\varphi_1,\varphi_3)}$. A precise formalization of this procedure can be done in terms of slice categories. For more details we refer the reader to Appendix~\ref{appendix-slice}.

When two perception pairs $(\varPhi,G)$, $(\varPsi,H)$ are considered as categories and a homomorphism $T:G\to H$ is fixed, each perception map $F$ from $(\varPhi,G)$ to $(\varPsi,H)$ is naturally associated with a functor between the two categories, taking each function $\varphi\in\varPhi$ to $F(\varphi)\in \varPsi$ and each morphism $g_{(\varphi_1,\varphi_2)}\in G$ to the morphism $T(g)_{\left(F(\varphi_1),F(\varphi_2)\right)}\in H$.
\end{remark}

%

\begin{definition}
Assume that $(\varPhi,G)$,$(\varPsi,H)$ are two perception pairs and that a homomorphism $T:G\to H$ has been fixed. If $F$ is a perception map from $(\varPhi,G)$ to $(\varPsi,H)$ with respect to $T$ and $F$ is non-expansive (i.e., $D_\varPsi\left(F(\varphi_1),F(\varphi_2)\right) \le D_\varPhi\left(\varphi_1,\varphi_2\right)$ for every $\varphi_1,\varphi_2\in\varPhi$), then $F$ is called a \emph{Group Equivariant Non-Expansive Operator (GENEO) associated with $T:G\to H$}.
\end{definition}

\begin{example}
\label{basicex}
As a reference for the reader, we give the following basic example of GENEO. Let $\varPhi$ be the set containing all $1$-Lipschitz functions from $X=S^2=\{(x,y,z)\in\R^3:x^2+y^2+z^2=1\}$ to $[0,1]$, and $G$ be the group of all rotations of $S^2$ around the $z$-axis. Let $\varPsi$ be the set containing all $1$-Lipschitz functions from $Y=S^1=\{(x,y)\in\R^2:x^2+y^2=1\}$ to $[0,1]$, and $H$ be the group of all rotations of $S^1$.
We observe that $(\varPhi,G)$ and $(\varPsi,H)$ are two perception pairs. Now, let us consider the map $F:\varPhi\to\varPsi$ taking each function $\varphi\in \varPhi$ to the function $\psi\in \varPsi$ defined by setting $\psi(\theta):=\frac{1}{\pi}\int_{-\pi/2}^{\pi/2} \varphi(\theta,\alpha)\ d\alpha$ (with $\theta,\alpha$ polar coordinates), and the homomorphism $T$ taking the rotation of $S^2$ of $\alpha$ radians around the $z$-axis positively oriented to the counter-clock rotation of $\alpha$ radians of $S^1$. We can easily check that $F$ is a perception map and a GENEO from $(\varPhi,G)$ to $(\varPsi,H)$, associated with the homomorphism $T$. In this example $F$ and $T$ are surjective, but an example where $F$ and $T$ are not surjective can be easily found, e.g. by restricting $\varPhi$ to the singleton $\bar\varPhi$ containing only the null function and $G$ to the trivial group $\bar G$ containing only the identical homomorphism.
\end{example}


We can study how GENEOs act on the natural pseudo-distances:

\begin{proposition}
\label{thmContraction}
If $F$ is a GENEO from $(\varPhi, G)$ to $(\varPsi, H)$ associated with $T:G\to H$, then it is a contraction with respect to the natural pseudo-distances $d_G$, $d_H$.
\end{proposition}
(The proof is in Appendix~\ref{AppendixProofs}.)\\

\subsubsection{Pseudo-metrics on $\GENEO\left((\varPhi,G),(\varPsi,H)\right)$}
\label{psmGENEO}

Let us denote by $\mathrm{GENEO}\left((\varPhi,G),(\varPsi,H)\right)$ the set of all GENEOs between two perception pairs $(\varPhi,G)$, $(\varPsi,H)$ associated with $T:G\to H$. We can endow this set with the following pseudo-distances $D_\mathrm{GENEO}$,
$D_\mathrm{GENEO,H}$.

\begin{definition}
If $F_1,F_2\in\mathrm{GENEO}\left((\varPhi,G),(\varPsi,H)\right)$, we set
\begin{align}
D_\mathrm{GENEO}\left(F_1,F_2\right)&:=\sup_{\varphi\in \varPhi}D_\varPsi\left(F_1(\varphi),F_2(\varphi)\right)\nonumber\\
D_\mathrm{GENEO,H}\left(F_1,F_2\right)&:=\sup_{\varphi\in \varPhi}d_H\left(F_1(\varphi),F_2(\varphi)\right).
\end{align}
\end{definition}

The next result can be easily proved by applying the inequality $d_H\le D_\Psi$ (see Theorem~\ref{t12}) and recalling that the supremum of a family of bounded pseudo-metrics is still a pseudo-metric.

\begin{proposition}
$D_\mathrm{GENEO}$ and $D_\mathrm{GENEO,H}$ are pseudo-metrics on $\mathrm{GENEO}\left((\varPhi,G),(\varPsi,H)\right)$. Moreover,  $D_\mathrm{GENEO,H}\le D_\mathrm{GENEO}$.
\end{proposition}

It would be easy to check that as a matter of fact $D_\mathrm{GENEO}$ is a metric.

This simple statement holds:
\begin{proposition}
\label{propzerof}
For every $F\in \mathrm{GENEO}\left((\varPhi,G),(\varPsi,H)\right)$ and every $\varphi \in \varPhi$:
$\|F(\varphi)\|_{\infty}\le \|\varphi\|_{\infty}+ \|F(\textbf{0})\|_{\infty}$, where $\textbf{0}$ denotes the function taking the value 0 everywhere.
\end{proposition}

(The proof is in Appendix~\ref{AppendixProofs}.)


\section{On the compactness and convexity of the space of GENEOs}
\label{sec:comp_conv}

In this section we show that, if the function spaces are compact and convex, then the space of GENEOs is compact and convex too. This property has important consequences from the computational point of view, since it guarantees that the space of GENEOs can be approximated by a finite set and that new GENEOs can be obtained by convex combination of preexisting GENEOs.

Several results in this section and in Section~\ref{sec:newPH} mimic the corresponding results in~\cite{FrJa16}, where the particular case $(\varPhi,G)=(\varPsi,H)$, $T=\mathrm{Id}:G\to H$ is considered. Note that considering different function spaces and different groups of equivariance is fundamental, as it allows one to compose operators hierarchically, in the same fashion as computational units are linked in an artificial neural network.

For the sake of conciseness, in the following we will set $\mathcal{F}^{\mathrm{all}}:=\mathrm{GENEO}\left((\varPhi,G),(\varPsi,H)\right)$.
We recall that we are assuming $\varPhi$ and $\Psi$ compact with respect to $D_\varPhi$ and $D_\Psi$, respectively.

\subsection{The space of GENEOs is compact with respect to $D_{\GENEO}$}

\begin{theorem}\label{t17}
$\mathcal{F}^{\mathrm{all}}$ is compact with respect to $D_{\GENEO}$.
\end{theorem}

(The proof is in Appendix~\ref{AppendixProofs}.)\\

\subsection{The set of GENEOs is convex}
Let $F_1, F_2, \dots , F_n$ be GENEOs from $(\varPhi,G)$ to $(\varPsi,H)$ associated with the homomorphism $T$.
Let $(a_1, a_2, \dots , a_n) \in \mathbb{R}^n$ with $\sum_{i = 1}^n|a_i|\leq 1$. Consider the function

\begin{equation}
F_{\Sigma}(\varphi) := \sum_{i = 1}^n a_i F_i(\varphi)
\end{equation}

from $\varPhi$ to the set $C^0(Y,\mathbb{R})$ of the continuous functions from $Y$ to $\mathbb{R}$, where $Y$ is the domain of the functions in $\Psi$.

\begin{proposition}
\label{propconvex}
If $F_{\Sigma}(\varPhi)\subseteq \varPsi$, then $F_{\Sigma}$ is a GENEO from $(\varPhi, G)$ to $(\varPsi, H)$ with respect to $T$.
\end{proposition}

(The proof is in Appendix~\ref{AppendixProofs}.)\\

\begin{theorem}
\label{thmconvex}
If $\Psi$ is convex, then the set of GENEOs from $(\varPhi, G)$ to $(\varPsi, H)$ with respect to $T$ is convex.
\end{theorem}

(The proof is in Appendix~\ref{AppendixProofs}.)

\subsection{GENEOs as agents in our model}
In our model the agents are represented by GENEOs. Indeed, each agent can be seen as a black box that receives and transforms data.
If a nonempty subset $\mathcal{F}$ of $\mathrm{GENEO}\left((\varPhi,G),(\varPsi,H)\right)$ is fixed, a simple pseudo-distance
$D_{\mathcal{F},\varPhi}(\varphi_1,\varphi_2)$ to compare two admissible functions $\varphi_1,\varphi_2\in\varPhi$ can be defined by setting
$D_{\mathcal{F},\varPhi}(\varphi_1,\varphi_2):=\sup_{F\in \mathcal{F}}\|F(\varphi_1)-F(\varphi_2)\|_\infty$. This definition expresses our assumption that the comparison of data strongly depends on the choice of the agents. However, we note that the computation of  $D_{\mathcal{F},\varPhi}(\varphi_1,\varphi_2)$ for every pair $(\varphi_1,\varphi_2)$ of admissible functions is computationally expensive. In the next section, we will see how persistent homology allows us to replace $D_{\mathcal{F},\varPhi}$ with a pseudo-metric $\mathcal{D}^{\mathcal{F},k}_{\mathrm{match}}$ that is quicker to compute, while still being stable and strongly invariant.
\section{A strongly group-invariant pseudo-metric induced by Persistent Homology}
\label{sec:newPH}

In this section, we show how Persistent Homology supports the definition of a strongly group invariant pseudo-metric on $\varPhi$, for which we prove some theoretical results.

We begin by recalling the stability of the classical pseudo-distance $d_{\mathrm{match}}$ between persistent Betti numbers functions (BPNs)  (cf. Definition \ref{def:matching}) with respect to the pseudo-metrics $D_\varPhi$ and $d_{\Homeo(X)}$. We assume the finiteness of PBNs \footnote{Though in our setting, the space $X$ is assumed to be compact, PBNs are not necessarily finite. For example, let us consider the set $X = \{0\} \cup \{ \frac{1}{n}, \ \text{with} \ n \in \mathbb{N} \}$ and $\varPhi = \{ \mathrm{Id}: X \longrightarrow X\}$. Even if $X$ is compact, every sublevel set $X_u= \{x \in X: x \le u \}$ with $u>0$ has infinite connected components, and hence the $0$th persistent Betti numbers function takes infinite value at every point $(u,v)$ with $0< u<v$.

We add the assumption on the finiteness of PBNs (i.e., the assumption that the persistent Betti numbers function of every $\varphi\in\varPhi$ takes a finite value at each point $(u,v)\in \Delta^{+}$) to get stability and discard pathological cases (for example the case that the set $\varPhi$ of admissible functions is the set of all maps from $X$ to $\R$).

Since the PBNs of the pseudo-metric space $(X,D_X)$ coincide with the persistent Betti numbers functions of its Kolmogorov quotient $\bar{X}$, the finiteness of the persistent Betti numbers functions can be obtained when $\bar{X}$ is finitely triangulable (cf. \cite{CeDFFeal13}).}. Then, the stability of $d_{match}$ with respect to $D_\varPhi$ easily follows from the stability theorem of the interleaving distance and the isometry theorem (cf. \cite{Ou15}).

\begin{theorem}
\label{t12}
If k is a natural number, $G_1\subseteq G_2\subseteq \Homeo_\varPhi(X)$ and $\varphi_1,\varphi_2 \in \varPhi$, then
\begin{equation}d_{\mathrm{match}} (r_k(\varphi_1),r_k(\varphi_2)) \le d_{\Homeo(X)} (\varphi_1,\varphi_2)
\le d_{G_2} (\varphi_1,\varphi_2)\le d_{G_1} (\varphi_1,\varphi_2)\le D_\varPhi (\varphi_1,\varphi_2).
\end{equation}
\end{theorem}

The proof of the first inequality $d_{\mathrm{match}} (r_k(\varphi_1),r_k(\varphi_2)) \le d_{\Homeo(X)}(\varphi_1,\varphi_2)$ in Theorem~\ref{t12} is based on the stability of $d_{match}$ with respect to $D_\varPhi$ and can be found in \cite{CeDFFeal13}. The other inequalities follow from the definition of the natural pseudo-distance.

\subsection[Strongly group invariant comparison]{Strongly group invariant comparison of filtering functions via persistent homology}
\label{sec:SGIC}

Let us consider a subset $\mathcal{F} \ne \emptyset$ of $\mathcal{F}^{\mathrm{all}}$. For every fixed $k$, we can consider the following pseudo-metric $\mathcal{D}^{\mathcal{F},k}_{\mathrm{match}}$ on $\varPhi$:
\begin{equation}
\mathcal{D}^{\mathcal{F},k}_{\mathrm{match}}(\varphi_1, \varphi_2):= \sup_{F \in \mathcal{F}} d_{\mathrm{match}}(r_k(F(\varphi_1)),r_k(F(\varphi_2)))
\end{equation}
for every $\varphi_1\varphi_2 \in \varPhi$, where $r_k(\varphi)$ denotes the $k$th persistent Betti numbers function with respect to the function $\varphi: X \rightarrow \mathbb{R}$.

In this work, we will say that a pseudo-metric $\hat{d}$ on $\varPhi$ is \textit{strongly G-invariant} if it is invariant under the action of $G$ with respect to each variable, that is, if $\hat{d}(\varphi_1, \varphi_2)=\hat{d}(\varphi_1 \circ g, \varphi_2)=\hat{d}(\varphi_1, \varphi_2 \circ g)=\hat{d}(\varphi_1 \circ g, \varphi_2 \circ g)$ for every $\varphi_1,\varphi_2 \in \varPhi$ and every $g \in G$.

\begin{remark}
\label{invdg}
It is easily seen that the natural pseudo-distance $d_G$ is strongly $G$-invariant.
\end{remark}

\begin{proposition}
\label{stronglyinv}
  $\mathcal{D}^{\mathcal{F},k}_{\mathrm{match}}$ is a strongly $G$-invariant pseudo-metric on $\varPhi$.
\end{proposition}
(The proof is in Appendix~\ref{AppendixProofs}.)\\

\subsection{Some theoretical results on the pseudo-metric $\mathcal{D}^{\mathcal{F},k}_{\mathrm{match}}$}
\label{sectionSTR}

At first we want to show that the pseudo-metric $\mathcal{D}^{\mathcal{F},k}_{\mathrm{match}}$ is stable with respect to both the natural pseudo-distance $d_G$ associated with the group $G$ and the distance $D_\varPhi$.

\begin{remark}
\label{r10}
Let $X$ and $Y$ be two homeomorphic spaces and  let $h:Y \rightarrow X$ be a homeomorphism. Then the persistent homology group with respect to the function $\varphi: X \rightarrow \mathbb{R}$ and the persistent homology group with respect to the function $\varphi \circ h: Y\rightarrow \mathbb{R}$ are isomorphic at each point $(u,v)$ in the domain.
Therefore we can say that the persistent homology groups and the persistent Betti numbers functions are invariant under the action of $\Homeo(X)$.
\end{remark}

\begin{theorem}
\label{t15}
If $\mathcal{F}$ is a non-empty subset of $\mathcal{F}^{\mathrm{all}}$, then
\begin{equation}\mathcal{D}^{\mathcal{F},k}_{\mathrm{match}} \le d_G \le D_\varPhi.
\end{equation}
\end{theorem}
(The proof is in Appendix~\ref{AppendixProofs}.)\\

The definitions of the natural pseudo-distance $d_G$ and the pseudo-distance $\mathcal{D}^{\mathcal{F},k}_{\mathrm{match}}$ come from different theoretical concepts. The former is based on a variation approach involving the set of all homeomorphisms in $G$, while the latter refers only to a comparison of persistent homologies depending on a family of group equivariant non-expansive operators. Given those comments, the next result may appear unexpected.

\begin{theorem}
\label{thrEqualityDFall_dmatch}
Let us assume that $\varPhi=\Psi$, every function in $\varPhi$ is non-negative, the $k$-th Betti number of $X$ does not vanish, and $\varPhi$ contains each constant function $c$ for which a function $\varphi \in \varPhi $ exists such that $0\le c \le \|\varphi\|_{\infty}$.
Then $\mathcal{D}^{\mathcal{F}^{\mathrm{all}},k}_{\mathrm{match}}=d_G$.
\end{theorem}
(The proof is in Appendix~\ref{AppendixProofs}.)\\

We observe that if $\varPhi$ is bounded, the assumption that every function in $\varPhi$ is non-negative is not quite restrictive. Indeed, we can obtain it by adding a suitable constant value to every admissible function.


\subsection{Pseudo-metrics induced by persistent homology}
\label{PM}

Persistent homology can be seen as a topological method to build new and easily computable pseudo-metrics for the sets $\varPhi$, $G$ and $\mathcal{F}^{\mathrm{all}}$.
These new pseudo-metrics $\Delta_\varPhi$, $\Delta_G$, $\Delta_\mathrm{GENEO}$ can be used as proxies for
$d_G$ (and hence $D_\varPhi$), $D_G$, $D_\mathrm{GENEO}$, respectively:
\begin{itemize}
  \item If $\varphi_1,\varphi_2\in\varPhi$, we can set $\Delta_\varPhi(\varphi_1,\varphi_2):=d_{\mathrm{match}}(r_k(\varphi_1),r_k(\varphi_2))$.
The stability theorem for persistence diagrams (Theorem~\ref{t12}) can be reformulated as the inequalities $\Delta_\varPhi\le d_G\le D_\varPhi$.

  \item If $g_1,g_2\in G$, we can set $\Delta_G(g_1,g_2):=
  \sup_{\varphi \in \varPhi}d_{\mathrm{match}}(r_k(\varphi \circ g_1),r_k(\varphi \circ g_2))$.
  From Theorem~\ref{t12} the inequality $\Delta_G\le D_G$ follows.
  \item If $F_1,F_2\in \mathcal{F}^{\mathrm{all}}$, we can set
  $\Delta_\mathrm{GENEO}\left(F_1,F_2\right):=
  \sup_{\varphi \in \varPhi}d_{\mathrm{match}}(r_k(F_1(\varphi)),r_k(F_2(\varphi)))$.
  From Theorem~\ref{t12} the inequalities $\Delta_\mathrm{GENEO}\le D_\mathrm{GENEO,H}\le D_\mathrm{GENEO}$ follow.
\end{itemize}

In particular, $\Delta_\varPhi$ and a discretized version of the pseudo-metric $\Delta_\mathrm{GENEO}$ will be used in the experiments described in \Cref{sec:results}. We underline that the use of persistent homology is a key tool in our approach: it allows for a fast comparison between functions and between GENEOs. Without persistent homology, this comparison would be much more computationally expensive.

\subsection{Approximating ${\mathcal{D}}^{\mathcal{F},k}_{\mathrm{match}}$ }

The next result will be of use for the approximation of $\mathcal{D}^{\mathcal{F},k}_{\mathrm{match}}$.


\begin{proposition}
\label{14}
Let $\mathcal{F},\mathcal{F}'\subseteq \mathcal{F}^\mathrm{all}$.
If the Hausdorff distance $$HD(\mathcal{F},\mathcal{F}'):=
\max\left\{
\sup_{F\in\mathcal{F}}\inf_{F'\in\mathcal{F}'} D_{\GENEOH}(F,F'),
\sup_{F'\in\mathcal{F}'}\inf_{F\in\mathcal{F}} D_{\GENEOH}(F,F')\right\}$$  is not larger than $\varepsilon$, then
\begin{equation}\left|\mathcal{D}^{\mathcal{F},k}_{\mathrm{match}}(\varphi_1, \varphi_2) - \mathcal{D}^{\mathcal{F}',k}_{\mathrm{match}}(\varphi_1, \varphi_2)\right|\le 2 \varepsilon
\end{equation}
for every $\varphi_1,\varphi_2 \in \varPhi$.
\end{proposition}
(The proof is in Appendix~\ref{AppendixProofs}.)\\

Since the compactness of the space $\mathcal{F}^\mathrm{all}$ guarantees we can cover $\mathcal{F}$ by a finite set of balls in $\mathcal{F}^\mathrm{all}$ of radius $\varepsilon$, centered at points of a finite set $\mathcal{F}'\subseteq \mathcal{F}$, the following proposition states that the approximation of $\mathcal{D}^{\mathcal{F},k}_{\mathrm{match}}(\varphi_1, \varphi_2)$ can be reduced to the computation of $\mathcal{D}^{\mathcal{F}',k}_{\mathrm{match}}(\varphi_1, \varphi_2)$, i.e. the maximum of a finite set of bottleneck distances between persistence diagrams, which are well-known to be computable by means of efficient algorithms.


\begin{proposition}
\label{corapprox}
Let $\mathcal{F}$ be a non-empty subset of $\mathcal{F}^{\mathrm{all}}$. For every $\varepsilon>0$, a finite subset $\mathcal{F}^{*}$ of $\mathcal{F}$ exists, such that
\begin{equation}
|\mathcal{D}^{\mathcal{F^{*}},k}_{\mathrm{match}}(\varphi_1, \varphi_2) - \mathcal{D}^{\mathcal{F},k}_{\mathrm{match}}(\varphi_1, \varphi_2)|\le \varepsilon
\end{equation}
for every $\varphi_1,\varphi_2 \in \varPhi$.
\end{proposition}

(The proof is in Appendix~\ref{AppendixProofs}.)\\

\begin{remark}
Theorem~\ref{t17} and the inequalities $\Delta_\mathrm{GENEO}\le D_\mathrm{GENEO,H}\le D_\mathrm{GENEO}$ stated in Subsection~\ref{PM} immediately imply that $\mathcal{F}^{\mathrm{all}}$ is compact also with respect to the topologies induced by $\Delta_\mathrm{GENEO}$ and   $D_\mathrm{\GENEO,H}$.
\end{remark}

\subsection{Beyond group equivariance}
We observe that while the definition of the natural pseudo-distance $d_G$ requires that $G$ has the structure of a group, the definition of $\mathcal{D}^{\mathcal{F},k}_{\mathrm{match}}$ does not need this assumption. In other words, our approach based on GENEOs can be used also when we wish to have equivariance with respect to a \emph{set} instead of a \emph{group} of homeomorphisms. This property is promising for extending the application of our theory to the cases in which the agent is equivariant with respect to each element of a finite set of homeomorphisms that is not closed with respect to composition and computation of the inverse.

\section{Validation on discrete function spaces}
\label{sec:results}

\begin{figure}[tb]
  \centering
  \includegraphics[width=.8\textwidth]{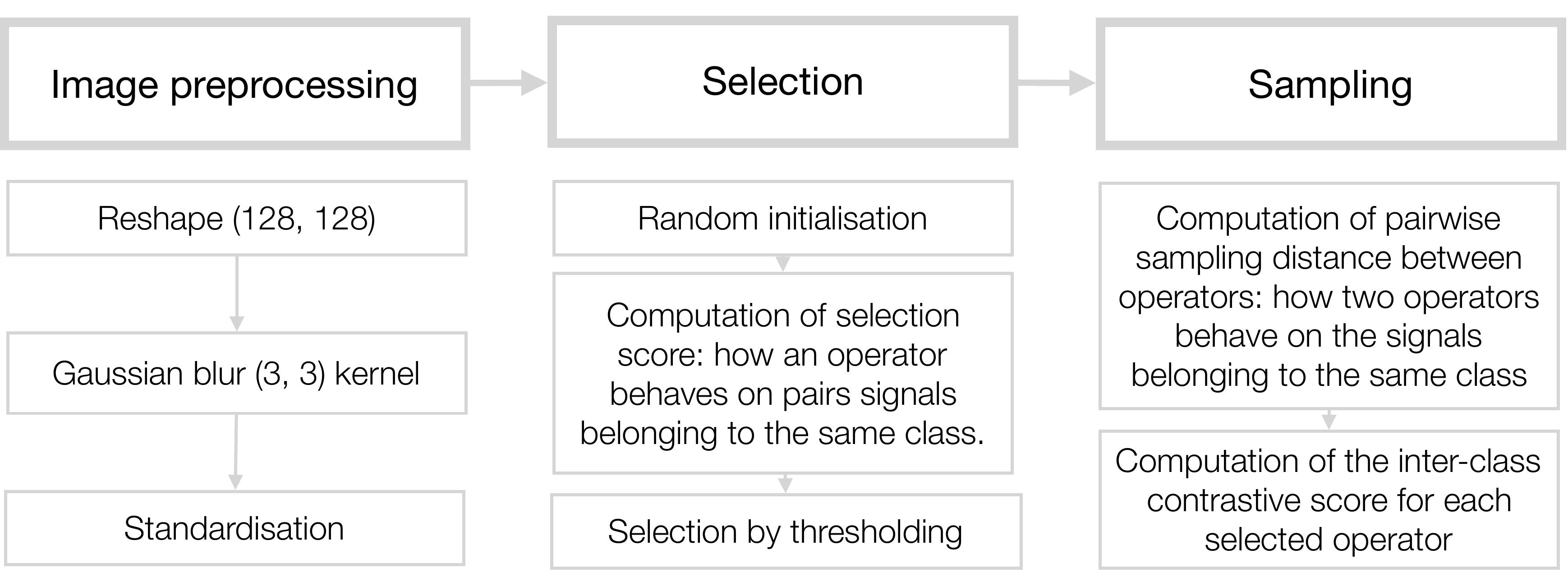}
  \caption{Experimental pipeline. In the preprocessing step images are reshaped, smoothed and standardised. Subsequently, operators are  initialised randomly. Thereafter, they are selected according to their output when evaluated on objects belonging to a chosen class. The final sampling step allows us to exclude operators that appear to be too similar, and thus redundant.\label{fig:pipeline}}
\end{figure}

In summary, we introduced above a theoretical framework allowing to describe an agent acting on data as a collection of suitable operators. We do that by representing data as points of a space of continuous functions with compact support. The density of such space makes the quest for suitable operators for the approximation of a given agent computationally complicated. For this reason, we chose to consider GENEOs: enforcing equivariance with respect to the action of a group causes the dimensionality of the search-space to collapse. Furthermore, in~\Cref{sec:mathmod}, we showed how GENEO spaces can be equipped with suitable metrics and respect properties that are essential in a machine learning context. The results concerning compactness and convexity make it possible to safely explore the space of GENEOs when operating on a labelled dataset. One of the main issue to be addressed when working in the proposed setting is the computability of metrics between operators. In~\Cref{sec:newPH} we show how metrics between GENEOs can be lower approximated via persistent homology. These results should be enough to guarantee approximability, efficacy and computability of GENEOs, when utilised to solve supervised tasks.

Our mathematical model and theorems are based on the assumption that data can be treated as points in a space of continuous functions. In this section, we test the validity of such results on classification of real-world datasets proceeding as follows. First we describe an algorithm allowing to select and sample GENEOs in order to learn the metric induced on a dataset by a labelling function. After that, we define the class of GENEOs we will use to study the MNIST, fashion-MNIST and CIFAR10 datasets. Selection and sampling are then used to approximate an agent able to express the underlying metric of these datasets by observing only $20$ or $40$ examples per class. Thereafter, we show how the metric learned through selection and sampling is still expressive when used to represent distances among validation samples transformed according to the equivariances of the GENEOs of choice. Finally, we use selected and sampled GENEOs to inject knowledge in an artificial neural network.

\subsection{Operators selection and sampling on labelled datasets}

We start from the assumption that data labelled with the same symbol share common features with respect to the agent we want to approximate. Thus, we suggest an algorithm for metric learning based on the metrics introduced on the space of GENEOs in~\Cref{sec:newPH}. Briefly, we start by selecting randomly a certain number of GENEOs. Afterwards, we compare them by taking advantage of the fact that their representation as persistence diagrams is invariant with respect to the action of $G$. These selected operators see those features that are common among the samples associated to the same label. Finally, always profiting from the property of the matching distance to be lower bound of the metric defined on the space of operators, we sample the operators in order to obtain a minimal set of non-redundant operators.

In symbols, let $\Phi = \left\{ \varphi_1, \dots, \varphi_n \right\}$ be a dataset equipped with a labelling function $l: \Phi\rightarrow I\in\mathbb{N}$. We assume that the dataset can be written as the disjoint union $\Phi = \sqcup_{i\in I} \Phi_i$ where $\Phi_i$ contains samples labelled by $i$.
Let $\mathcal{F}$ be the space of operators that will act on the samples. We begin by randomly sampling $N$ candidate operators in $\mathcal{F}$, let us denote them as the set $\mathcal{C} = \left\{F_k\right\}_{k\in\{1,\dots, N\}}$. We then select those operators that consider as similar the objects belonging to the same class. Let us consider the samples in $\Phi_l$, for each of the candidate operators $F\in\mathcal{C}$, we define the label-dependent value

$$
s_{l}\left(F \right) = \max_{\varphi_{i}^{l},\varphi_{j}^l}d_{\mathrm{match}}\left(r_1\left(F\left(\varphi_{i}^{l}\right)\right),r_1\left(F\left(\varphi_{j}^{l}\right)\right)\right).
$$ 
A candidate operator $F$ is \emph{selected} if $s_l\left(F\right)$ is smaller than a fixed threshold $\epsilon$ for every $l$. Let us denote by $\mathcal{S}$ the set of selected operators. In practice, we will show how few examples per class are enough to select operators able to grasp salient topological-geometrical features from the example samples, and can be consequently used to compute reasonable distances between new validation samples.

The selection criterion does not guarantee that the operators are maximally diverse, when evaluated within and in-between classes. The important advantage of working on metric spaces is that we can now sample the elements of $\mathcal{S}$ to avoid storing operators that would focus on the same or similar characteristic. To this end, given a class $l$, we define the distance between two operators $F_{p}$ and $F_{q}$ (cf. Subsection \ref{PM})

$$
\Delta^{l}_\mathrm{GENEO}\left(F_{p},F_{q}\right):=
  \max_{\varphi_{i}^{l}}d_{\mathrm{match}}\left(r_k\left(F_{p}\left(\varphi_{i}^{l}\right)\right),r_k\left(F_{q}\left(\varphi_{i}^{l}\right)\right)\right).$$

For every label $l$, we sort the pairs $\left(F_{p},F_{q}\right)$ in ascending order of $\Delta^{l}_\mathrm{GENEO}$, and assign to each pair of operators its index in the sorted list of distances. We then define the \textit{interclass contrastive score} of the pair $\left(F_{p},F_{q}\right)$ as the sum of its indices over all classes. Finally, we remove from $\mathcal{S}$ redundant operators, i.e. we select only one operators for pairs whose score is below a fixed threshold $t$.

Finally, two objects $\varphi_1$ and $\varphi_2$ can be compared by computing the strongly $G$-invariant pseudo-metric $\mathcal{D}^{\mathcal{S}}_{\mathrm{match}}(\varphi_1, \varphi_2)$, defined in \Cref{sec:newPH}.

\begin{figure}[tb]
  \centering
  \includegraphics[width=.55\textwidth]{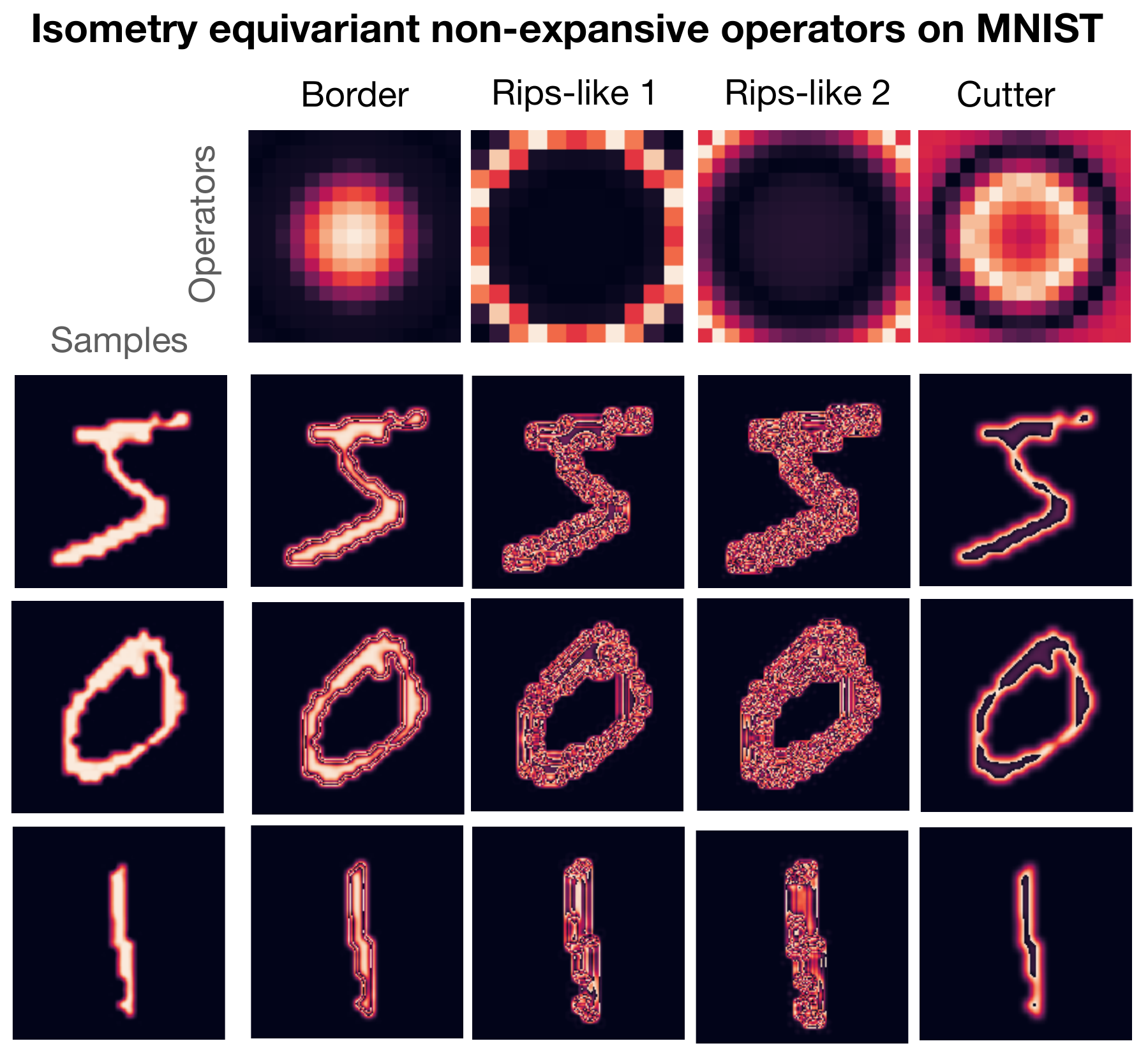}
  \caption{Selected IENEOs on the MNIST dataset. By considering the metrics defined on the space of GENEOs, we select operators able to recognise MNIST digits belonging to the same class. Albeit the operator space is constrained by the symmetries induced by the equivariance with respect to the group of planar isometries, we observe that the selected operators are expressive and have a clear topological interpretation. Among other functions, these operators are specialised in border detection and voids (cycles in topological terms) filling, similarly to the classical Vietoris-Rips construction~\cite{zomorodian2010fast}. Finally, the cutter operator--depicted in the last column--disconnects the image approximately according to its distance transform~\cite{fabbri20082d}. \label{fig:ieneos_mnist}}
\end{figure}

\subsection{Isometry equivariant non-expansive operators}

One of the main strength of convolutional neural networks is the natural equivariance of the convolution operator with respect to the group of planar translations. However, oftentimes when working with images or volumes, invariance with respect to other transformations such as rotations or reflexions can be important. In what follows we define a parametric family of non-expansive operators which are equivariant with respect to Euclidean plane isometries.

Given $\sigma>0$ and $\tau\in\R$, we consider the $1$-dimensional Gaussian function with width $\sigma$ and centre $\tau$
$$g_{\tau}(t):=e^{-\frac{\left(t-\tau\right)^2}{2\sigma^2}},$$
where $g_{\tau}:\R\to\R$. For a positive integer $k$, we take the set $S$ of the $2k$-tuples $(a_1,\tau_1,\ldots,a_k,\tau_k) \in \R^{2k}$ for which $\sum_{i=1}^{k}a_i^2=\sum_{i=1}^{k}\tau_i^2=1$. $S$ is a submanifold of $\R^{2k}$.

For each $p=(a_1,\tau_1,\ldots,a_k,\tau_k)\in S$, we then consider the function $G_p:\R^2\to\R$ defined as
$$G_p(x,y):=\sum_{i=1}^{k}a_i g_{\tau_i}\left(\sqrt{x^2+y^2}\right).$$

If we denote by $F_p$ the convolutional operator mapping each continuous function with compact support $\varphi:\R^2\to\R$ to the continuous and with compactly supported function $\psi:\R^2\to\R$ defined as
$$
\psi(x,y):=\int_{\R^2}\varphi(\alpha,\beta)\cdot\frac{G_p(x-\alpha,y-\beta)}{\|G_p\|_{L^1}}\ d\alpha\ d\beta.
$$
Then, the operator $F_p$ is a group equivariant non-expansive operator with respect to the group $I$ of Euclidean plane isometries. We call $F_p$ a IENEO (Isometry Equivariant Non-Expansive Operator).

The IENEO $F_p$ is parametric with respect to the $2k$-tuple $p=(a_1,\tau_1,\ldots,a_k,\tau_k)\in S$. Therefore, we define a parametric family of IENEOs $\mathcal{F} = \{F_{p}\}_{p \in S}$.

\begin{figure}[tb]
  \centering
  \includegraphics[width=.9\textwidth]{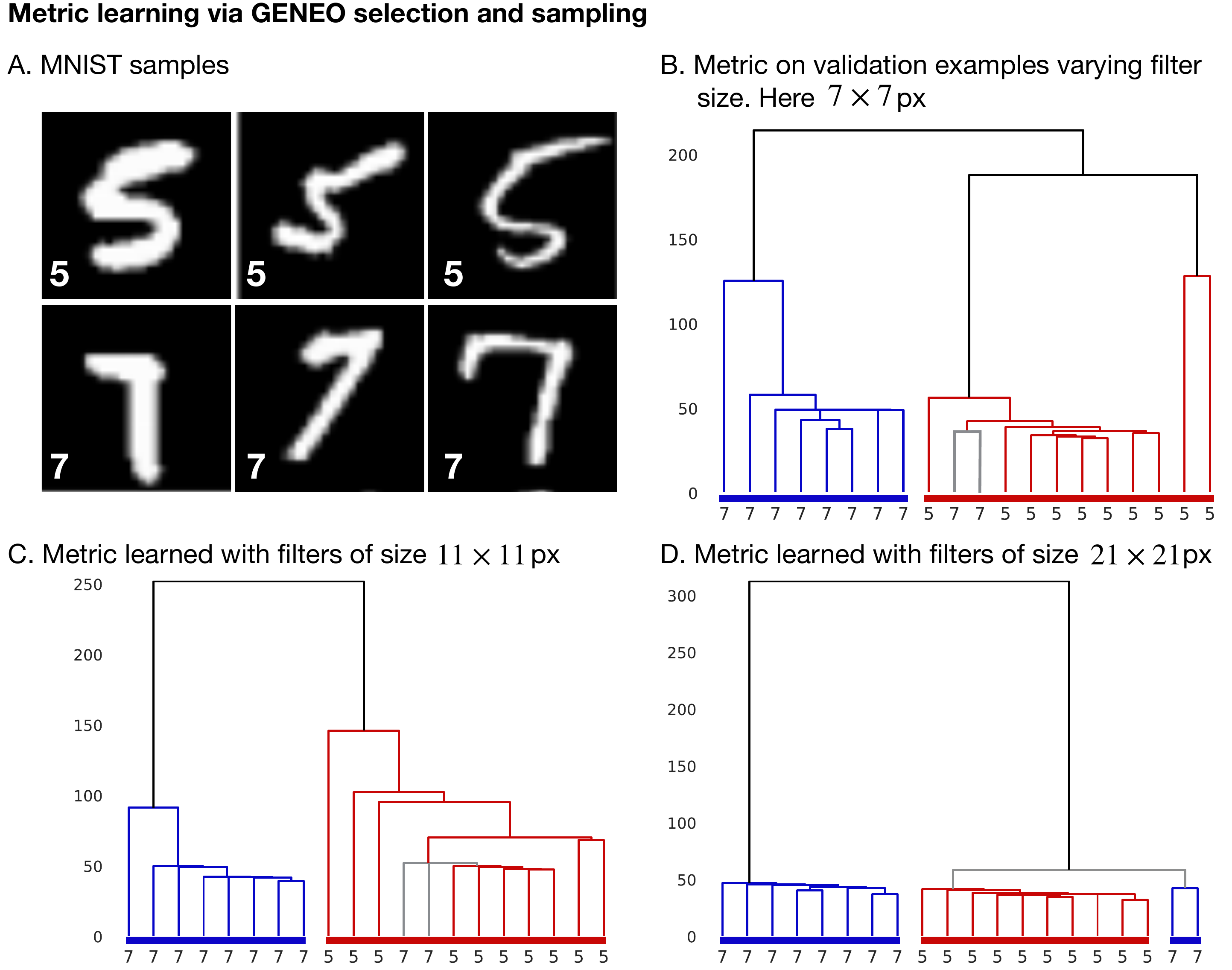}
  \caption{Metric learning via IENEOs selection and sampling. A. Samples from the classes $7$ and $5$ from the MNIST dataset. Panels B, C and D show the hierarchical clustering obtained by using selected and sampled IENEOs of different dimensions to measure the distance between validation samples belonging to the two considered classes. For all filter sizes $500$ operators were randomly initialised, then selected and sampled. We observe how samples belonging to the two classes are clearly separated by filters of all the considered dimensions.\label{fig:metric_MNIST}}
\end{figure}

\subsection{Applications}

We are now ready to utilise the selection and sampling strategy to find operators able to recognise samples belonging to the same class in a discrete dataset. We propose three different applications of our model. First we select and sample operators on two-classes subsets of the MNIST, fashion-MNIST and CIFAR10 datasets, we evaluate the validity of the learned metric by computing pairwise distances of validation samples according to the selected and sampled operators. Let us denote these operators by $\mathcal{S}$. Second, we evaluate on the MNIST dataset the capacity of the operators in $\mathcal{S}$ to discriminate validation examples that have been transformed with random planar isometries. Finally, we use $\mathcal{S}$ to initialise the filters of a convolutional layer and a dense architecture to classify the samples belonging to the classes the IENEOs where selected and sampled on.

\subsubsection{Image preprocessing}
Images are preprocessed according to the pipeline described in the first column of~\Cref{fig:pipeline}. Every image $I$ is first reshaped to size $(128,128)$, then blurred with a $3\times 3$ Gaussian kernel and finally standardised as $I_s = \frac{I-\textrm{mean}(I)}{\textrm{std}(I)}$. The same preprocessing is applied in all experiments and to all datasets.

\subsubsection{Metric learning through selection and sampling}

\begin{figure}[tb]
  \centering
  \includegraphics[width=\textwidth]{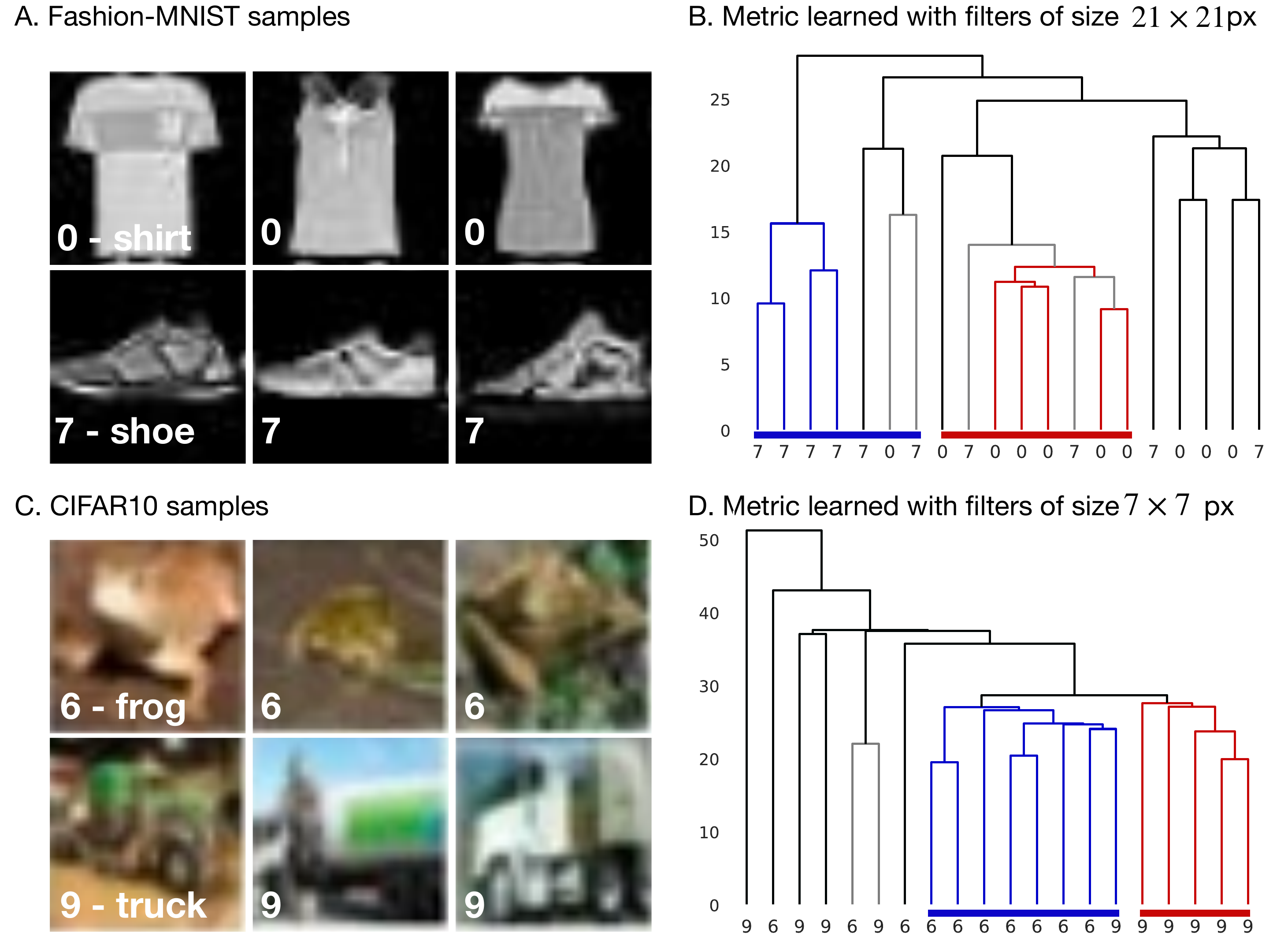}
  \caption{IENEO metric learning on fashion-MNIST and CIFAR10. After selecting and sampling $500$ randomly initialised IENEOs on $20$ examples per class, we evaluated the metric encoded by the selected operators on a validation set consisting of $10$ examples per class.\label{fig:metric_fashionMNISTCIFAR10}}
\end{figure}

\begin{figure}[h!]
  \centering
  \includegraphics[width=\textwidth]{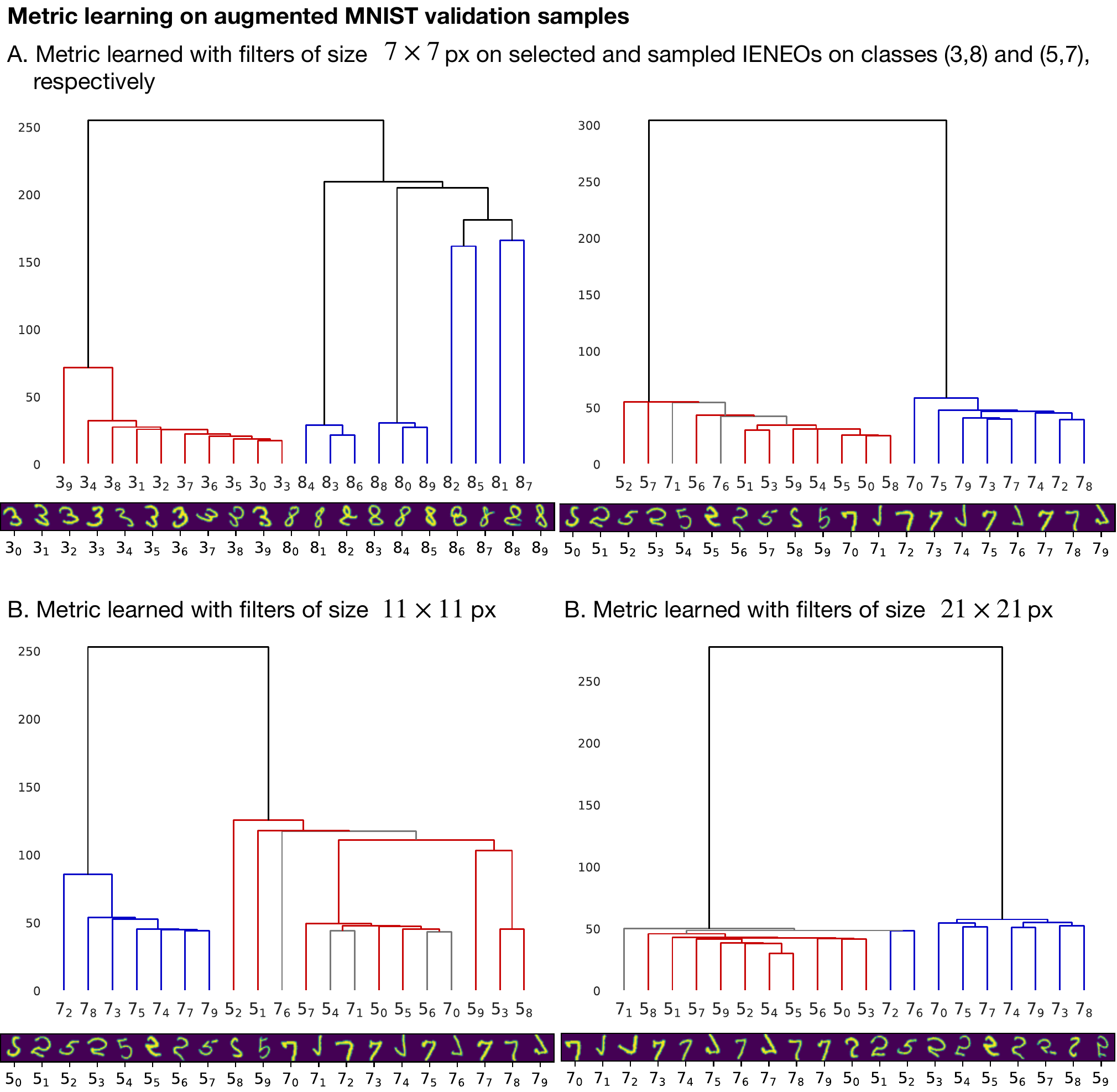}
  \caption{The metric obtained by considering selected and sampled IENEOs can be used to cluster samples transformed according to the group of equivariances (in this case planar isometries) of the family of operators of choice. In panel A we compare the cluster of transformed validation samples obtained on two different pairs of MNIST classes. The dendrogram on the right of panel A, panels B and C show how a variation in the size and number of Gaussian components of the IENEO affects the clustering of validation samples randomly transformed through planar isometries. \label{fig:augmented}}
\end{figure}

Metric learning is a natural application in the framework we describe. Indeed, operators that have been selected on labelled examples should be able to grasp geometrical and topological features that are shared among the examples belonging to the same class. Afterwards, selected and sampled operators $F_i\in\mathcal{S}$ can be used to measure distances between pairs of validation samples as
\begin{equation}\label{eq:metric}
  d_{\mathcal{S}}\left(\varphi, \varphi^\prime\right) = \max_{F\in \mathcal{S}}d_{\mathrm{match}}\left(r_1\left(F\left(\varphi\right)\right),r_1\left(F\left(\varphi^\prime\right)\right)\right).
\end{equation}

This choice implies that two samples $\varphi$ and $\varphi^\prime$ will have distance $0$, and hence are considered the same by the collection of selected operators (agent), only if every operator in $\mathcal{S}$ \textit{sees} them as identical. Note also that $d_{\mathcal{S}}$ is invariant with respect to the action of the group of planar isometries. This invariance is naturally inherited by the usage of $d_{\textrm{match}}$. After computing the pairwise distance between validation examples, we use hierarchical clustering~\cite{langfelder2007defining} to visualise how samples have been organised by the metric as a dendrogram.

For every dataset $\Phi\in\{\textrm{MNIST}, \textrm{fashion-MNIST}, \textrm{CIFAR10}\}$, we select a subset $\Phi_{l_i,l_j}$ of samples belonging to two classes. We start by randomly initialising a parametrised family of IENEOs (of cardinality $500$ or $750$ in the experiments that follow). Afterwards, a small number--typically $20$ or $40$--of samples per class are randomly chosen. These samples are then used to select common within-class geometrical and topological features by selection and sampling. The threshold for the selection algorithm is set to $\tau = 1.5$ and the threshold $t$ for sampling is defined as the $75th$ percentile of all contrastive scores. These parameters are fixed and used in all the following experiments.

We first studied the efficacy of selection and sampling on a binary classification task on the MNIST dataset. After selecting samples belonging to two randomly selected classes of MNIST, we chose $20$ random samples per class to be used as examples in the selection and sampling algorithm. Sampled and selected IENEOs are then used to compute the pairwise distances of $10$ validation samples per class and generate the dendrogram in panel B of~\Cref{fig:metric_MNIST}. We reproduced three times the same experiment by varying the size and the number of $1$-dimensional Gaussians used to initialise the IENEOs. In particular, we considered sizes $s \in \{7,11,21\}$. The number of Gaussians was chosen according to the size as $\frac{s}{2} + 1$ and rounded to the nearest integer. The dendrograms resulting from this manipulation are depicted in panels B, C, D in~\Cref{fig:metric_MNIST}.

Successively, we applied the same strategy and parameters to the fashion-MNIST and CIFAR10 datasets, obtaining the results in~\Cref{fig:metric_fashionMNISTCIFAR10}.

\subsubsection{Validation on augmented samples}
This application aims at testing the aforementioned equivariance of the distance $d_{\mathcal{S}}$ defined in Equation~\ref{eq:metric}. To do this, we consider a set of operators selected and sampled on non-transformed samples, while we transform the set of validation samples by applying a random transformation among translations, rotations and reflections parametrised as follows:
\begin{enumerate}
  \item rotations are selected randomly to be between $1$ and $30$ degrees;
  \item translations can be in both the $x$ and $y$-axis directions in a range between $1$ and $2$ pixels;
  \item reflections are computed randomly with respect to one of the two axes.
\end{enumerate}

The transformed samples along with the dendrograms obtained by considering the metric induced by the selected and sampled operators are shown in~\Cref{fig:augmented}.

\subsubsection{Knowledge injection}

As a final application, we discuss the possibility of using selected and sampled operators $\mathcal{S}$ as fixed feature extractor for a simple artificial neural network model. We do that by using the elements of $\mathcal{S}$ to initialise non-trainable filters of a convolutional layer. On top of this layer, we use two fully-connected layers, the first with ReLu~\cite{nair2010rectified} and the latter softmax activations, two classify samples from pairs of classes of MNIST, fashion-MNIST and CIFAR10 datasets. Then we compare the performance of the classifier operating with the IENEO-initialised filters, with an identical architecture whose filters were initialised randomly with Glorot initilisation~\cite{glorot2010understanding}. The architecture of the model and the performance are shown in~\Cref{fig:conv}.

\begin{figure}[htbp]
  \centering
  \includegraphics[width=.9\textwidth]{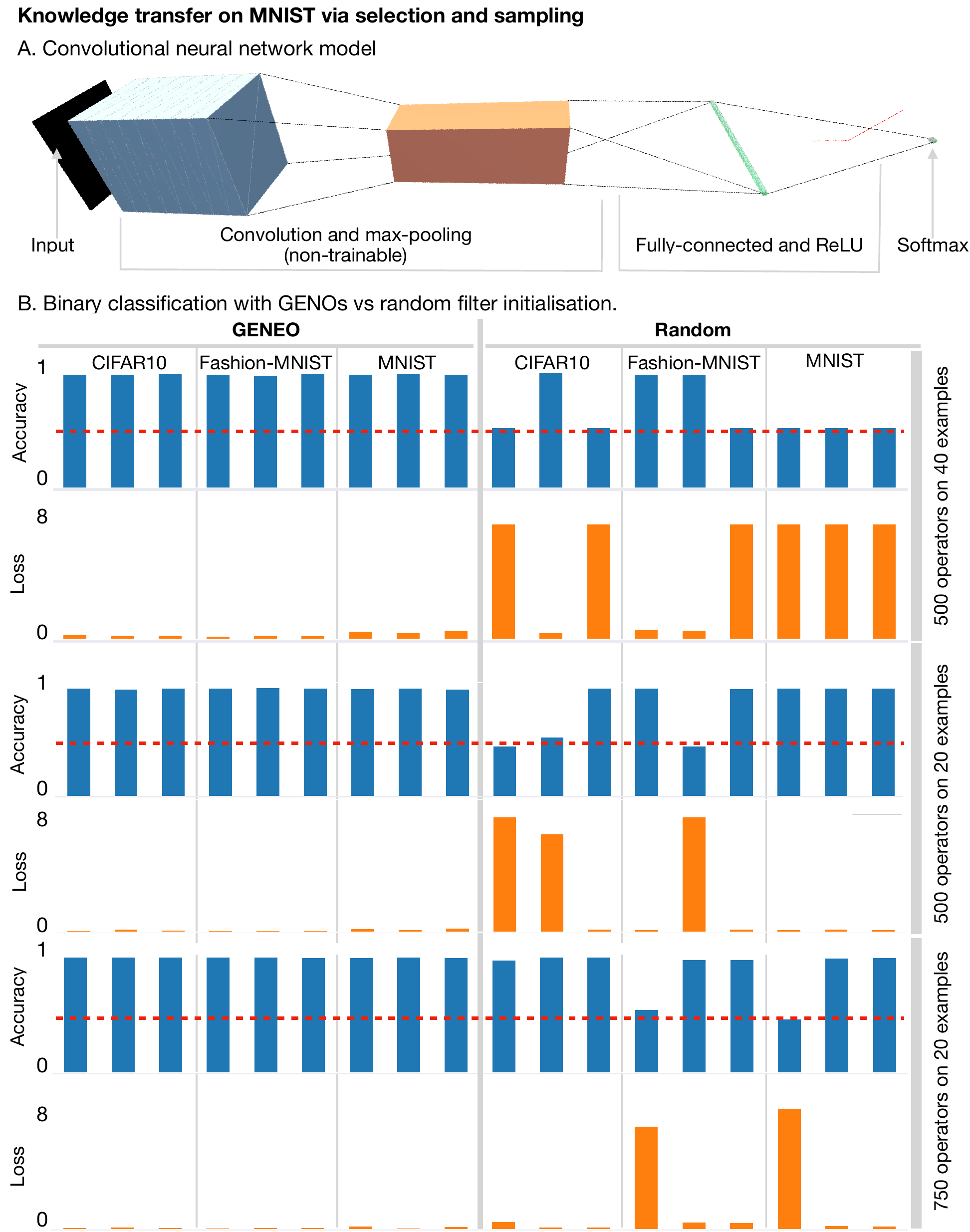}
  \caption{IENEOs versus random: Comparison between the performance obtained by classifying two classes of MNIST, fashion-MNIST and CIFAR10 with a dense classifier fed with convolutional filters obtained by selecting and sampling IENEOs. (A) The convolutional neural network architecture used in this experiment. The two first layers are a convolutional and maxpooling layer, respectively. The parameters of these layers are fixed. Two fully-connected layers counting $64$ and $2$ units are trained to classify the two selected classes. The first fully-connected layer uses a rectified linear unit (ReLU) and the second a softmax as activation functions. (B) Loss value and accuracy for the  validation set. In each row a different initialisation in terms of number of IENEOs and examples used during seletion and sampling is used.\label{fig:conv}}
\end{figure}

\section{Discussion and conclusions}
\label{sec:concls}

The first contribution of this paper consists in giving a novel, formal and sound mathematical framework for machine learning, based on the study of metric and topological properties of operator spaces acting on function spaces. This approach is dual to the classical one: instead of focusing on data, our approach focuses on suitable operators defined on the functions that represent the data. Of all possible type of operators, we study the space of non-expansive, group equivariant operators (GENEOs). When building a machine learning system, choosing to work on a space of operators equivariant with respect to specific transformations allows us to inject in the system pre-existing knowledge. Indeed, the operators will be blind to the action of the group on the data, hence reducing the dimensionality of the space to be explored during optimisation. The choice of working with non-expansive operators is justified both by the possibility of proving the compactness of the spaces of GENEOs (under the assumption of compactness of the spaces of measurements), and by the fact that in practical applications we are usually interested in operators that compress the information we have as an input.
The rationale of our approach is based on the assumption that the main interest in machine learning does not consist in the analysis and the approximation of data, but in the analysis and the approximation of the observers looking at the data. A simple example can make this idea clearer: if we consider images representing skin lesions, we are not mainly interested in the images per se but rather in approximating the judgement given by the physicians about such images.

Presenting our mathematical model, we first show how the space of GENEOs is suitable for machine learning. By using pseudo-metrics, we define a topology on the space of GENEOs which is induced by the one we define on the function space of data. We build the necessary machinery to define maps between GENEOs whose groups of equivariance are different. This definition is fundamental, because it allows one to compose operators hierarchically, in the same fashion as computational units are linked in an artificial neural network. Thereafter, by taking advantage of known and novel results in persistent homology, we prove compactness and convexity of the space of GENEOs under suitable hypotheses.
Moreover and importantly, we show how the suggested framework can be used to study operators that are equivariant with respect to set of transformations, rather than groups. In particular, we observe that the pseudo-metric $\mathcal{D}^{\mathcal{F},k}_{\mathrm{match}}$ defined in Subsection~\ref{sec:SGIC} can be used also in the case that the operators in $\mathcal{F}$ are equivariant with respect to a \emph{set} instead of a \emph{group} of homeomorphisms. This possibility appears to be promising for future research. It is important to stress the use of persistent homology in our model: the metric comparison of GENEOs is a key point in our approach and persistent homology allows for a fast comparison of functions, so allowing for a fast comparison of GENEOs.

We give two algorithms that allow to select and sample from a space of operators given a dataset labelled for a classification task. These procedures allow to first select a subset of operators belonging to a certain GENEOs space, that give meaningful representation of the data with respect to their labelling, always invariant under the transformations induced by the action of $G$. Thenceforth, the sampling algorithm allows to eliminate redundant operators. These two strategies are used to perform metric learning and kernel on MNIST and fashion-MNIST. In addition, we show how convolutional filters initialised by selecting and sampling on few samples effectively grasp useful knowledge, that can be utilised to classify the remainder of the samples, for instance by a dense classifier.

Our forward-looking goal is the one of defining a novel artificial neural network model based on functional modules. Modules would be more complex computational units than the standard artificial neuron. The core of each module would be a GENEO, thus each module would be defined a priori to be equivariant with respect to a set of transformations. On one hand, this choice would allow us to dramatically reduce the dimensionality of the manifold to be studied during optimisation. On the other hand, choosing the transformation equivariances to be respected at each layer would allow us to inject knowledge in the networks before training, and would assure that information is not acquired by relying on unwanted noisy regularities in the training data.
Module networks would learn optimal transformations of the data to achieve a task, rather than operating on data themselves.

Module networks could be built by composing modules hierarchically and knowledge could be injected in the model by engineering the proper set of equivariances. These transformations would be easily interpretable and could offer a rigorous way to compare learning dynamics of different architectures during optimisation. In particular, we are investigating the possibility to generalize capsule networks~\cite{Hi11,Hi17} and modify the dynamic routing algorithm, by using the metrics on the space of GENEOs to update the connectivity strength between modules.

We conclude by observing that several interesting problems and new lines of research naturally arise in our mathematical model. First of all some sets of GENEOs appear to have a structure of a Lie group and a Riemannian manifold: these structures seem worth study and analysis. Secondly, new methods for building GENEOs should be developed, in order to get good approximations of the spaces of GENEOs for given equivariance groups and function spaces.
We plan to devote further research to these issues.

\newpage

\begin{appendices}
\section{Additional propositions}
\label{appendix-D_X}

\begin{proposition}
The function $D_X$ is an extended pseudo-metric on $X$.
\end{proposition}
\begin{remark}
We recall that a pseudo-metric is just a distance $d$ without the property: if $d(a,b)=0$, then $a=b$.
\end{remark}
\begin{proof}
\begin{itemize}
\item  $D_X$ is obviously symmetrical.
\item The definition of $D_X$ immediately implies that $D_X(x,x)=0$ for any $x\in X$.
\item The triangle inequality holds, since
\begin{align*}
 D_X(x_1,x_2) & =\sup_{\varphi \in \varPhi}|\varphi (x_1)- \varphi (x_2)  |\\
 & \le \sup_{\varphi \in \varPhi}(|\varphi (x_1) - \varphi (x_3)| + |\varphi(x_3) - \varphi (x_2)  |)\\
 & \le \sup_{\varphi \in \varPhi}|\varphi (x_1) - \varphi (x_3)  | + \sup_{\varphi \in \varPhi} |\varphi (x_3) - \varphi (x_2)  |\\
 &= D_X(x_1,x_3) + D_X(x_3,x_2)
\end{align*}
for any $x_1,x_2,x_3 \in X$.
\end{itemize}

\end{proof}

\begin{proposition}
\label{stimaphi}
If $\varPhi$ is totally bounded, then for any $\delta >0$ there exists a finite subset $\varPhi_{\delta}$ of $\varPhi$ such that
\begin{equation*}
\left|\sup_{\varphi \in \varPhi}|\varphi(x_1)-\varphi(x_2)| - \max_{\varphi \in \varPhi_{\delta}}|\varphi(x_1) - \varphi(x_2)|\right| < 2\delta
\end{equation*}
for every $x_1, x_2\in X$.
\end{proposition}

\begin{proof}
Let us fix $x_1,x_2 \in X.$
Since $\varPhi$ is totally bounded, we can find a finite subset $\varPhi_{\delta}=\{ \varphi_1, \dots, \varphi_n\}$ such that for each $\varphi \in \varPhi$ there exists $\varphi_i \in \varPhi_{\delta}$, for which $\|\varphi-\varphi_i\|_{\infty} < \delta$. It follows that for any $x \in X, \ |\varphi(x) -\varphi_i(x)| < \delta$.
Because of the definition of supremum of a subset of the set $\mathbb{R}^+$ of all positive real numbers,  for any $\varepsilon>0$ we can choose a $\bar{\varphi} \in \varPhi$ such that $$\sup_{\varphi \in \varPhi}|\varphi(x_1)- \varphi(x_2)| -|\bar{\varphi}(x_1)- \bar{\varphi}(x_2)| \le \varepsilon.$$
Now, if we take an index $i$, for which  $\|\bar{\varphi}-\varphi_i\|_{\infty} < \delta$, we have that:
\begin{align*}
|\bar{\varphi}(x_1)- \bar{\varphi}(x_2)| & =|\bar{\varphi}(x_1)-\varphi_i(x_1) + \varphi_i(x_1)- \varphi_i(x_2) + \varphi_i(x_2)- \bar{\varphi}(x_2)|  \\
 &  \le|\bar{\varphi}(x_1)-\varphi_i(x_1)| +| \varphi_i(x_1)- \varphi_i(x_2)| + |\varphi_i(x_2)- \bar{\varphi}(x_2)|\\
 & < | \varphi_i(x_1)- \varphi_i(x_2)| +2 \delta \\
  & \le \max_{\varphi_j\in \varPhi_{\delta}}| \varphi_j(x_1)- \varphi_j(x_2)| +2 \delta.
\end{align*}
Hence,
\begin{equation*}
\sup_{\varphi \in \varPhi}|\varphi(x_1)- \varphi(x_2)|- \varepsilon  < |\bar{\varphi}(x_1)- \bar{\varphi}(x_2)| < \max_{\varphi_j\in \varPhi_{\delta}}| \varphi_j(x_1)- \varphi_j(x_2)| +2 \delta.
\end{equation*}

Finally, as $\varepsilon$ goes to zero, we have that
\begin{equation*}
\sup_{\varphi \in \varPhi}|\varphi(x_1)- \varphi(x_2)|< \max_{\varphi_j\in \varPhi_{\delta}}| \varphi_j(x_1)- \varphi_j(x_2)| +2 \delta.
\end{equation*}
On the other hand, since $\varPhi_{\delta} \subseteq \varPhi$:
\begin{equation*}
\sup_{\varphi \in \varPhi}|\varphi(x_1)- \varphi(x_2)| > \max_{\varphi_j\in \varPhi_{\delta}}| \varphi_j(x_1)- \varphi_j(x_2)| - 2\delta.
\end{equation*}

Therefore we proved the statement.

\end{proof}

\begin{proposition}
\label{propD_Gpseudometric}
The function $D_G$ is a pseudo-metric on $G$.
\end{proposition}

\begin{proof}
\begin{itemize}
 \item The value $D_G(g_1,g_2)$ is finite for every $g_1,g_2 \in G$ , because $\varPhi$ is compact and hence bounded. Indeed, a finite constant
 $L$ exists such that $\left \|\varphi \right \|_{\infty} \le L$ for every $\varphi \in \varPhi$. Hence, $\left \|\varphi \circ g_1 - \varphi \circ g_2  \right \|_{\infty}
 \le \left \|\varphi \right \|_{\infty}  + \left \|\varphi \right \|_{\infty} \le 2L $ for any $\varphi \in \varPhi$ and any $g_1,g_2 \in G$, since  $\varphi \circ g_1 , \varphi \circ g_2 \in \varPhi$.
 This implies that $D_G(g_1,g_2) \le 2L$ for every $g_1,g_2 \in G$.
\item  $D_G$ is obviously symmetrical.
\item The definition of $D_G$ immediately implies that $D_G(g,g)=0$ for any $g\in G$.
\item The triangle inequality holds, since
\begin{align}
 D_G(g_1,g_2) & =\sup_{\varphi \in \varPhi}\left \|\varphi \circ g_1 - \varphi \circ g_2  \right \|_{\infty}\nonumber\\
 & \le \sup_{\varphi \in \varPhi}(\left \|\varphi \circ g_1 - \varphi \circ g_3  \right \|_{\infty} + \left \|\varphi \circ g_3 - \varphi \circ g_2  \right \|_{\infty})\\
 & \le \sup_{\varphi \in \varPhi}\left \|\varphi \circ g_1 - \varphi \circ g_3  \right \|_{\infty} + \sup_{\varphi \in \varPhi} \left \|\varphi \circ g_3 - \varphi \circ g_2  \right \|_{\infty}\nonumber\\
 &= D_G(g_1,g_3) + D_G(g_3,g_2)\nonumber
\end{align}
for any $g_1,g_2,g_3 \in G$.
\end{itemize}
\end{proof}

\section{Our approach in terms of slice categories}
\label{appendix-slice}
In this section, we will apply the concept of slice category to our framework in order to formalize the concept of perception pairs, which are considered as subcategories of a larger category denoted by $\faktor{\text{PMet}}{(\R,d_e)}$, as we explain further. Moreover we explore the link between GENEOs and functors between categories of this kind.

Let $\text{PMet}$ be the category whose objects are pseudo-metric spaces and morphisms are the continuous functions between them. Let us fix the space $(\R,d_e)$, that is the real line equipped with the usual Euclidean metric, and consider the slice category over $(\R,d_e)$.

Now we recall the definition of slice category:

\begin{definition}
The slice category $\faktor{C}{c}$ of a category $C$ over an object $c \in C$ has
\begin{itemize}
    \item objects that are all arrows $f \in C$ such that $\text{cod}(f)=c$,
    \item morphisms that are all triples $g_{f,f'}:=(f, g , f' )$ where  $f: X \longrightarrow c$ and $f': X' \longrightarrow c$ are two objects of $\faktor{C}{c}$, $g: X \longrightarrow X'$ is a morphism of $C$ such that $f = f' \circ g$; $\text{dom}(g_{f,f'})=f \text{\ and } \text{cod}(g_{f,f'})= f'$.
\end{itemize}
The slice category is a special case of a comma category.
\end{definition}

\begin{remark}
There is a forgetful functor $U_c: \faktor{C}{c} \longrightarrow C$ which maps each object $f: X \longrightarrow c$ to its domain $X$ and each  morphism $g_{f,f'}$ between $f: X \longrightarrow c$ and $f': X' \longrightarrow c$ to the morphism $g: X \longrightarrow X'$.
\end{remark}

We are going to associate a perception pair $(\varPhi,G)$ with a subcategory $C(\varPhi,G)$ of $\faktor{\text{PMet}}{(\R,d_e)}$ defined as follows:
\begin{itemize}
    \item the objects of $C(\varPhi,G)$ are the elements of $\varPhi$;
    \item the arrows of $C(\varPhi,G)$ are the triples $( f, g , f \circ g )$ , where $f \in \varPhi$ and $g \in G$.
\end{itemize}
We observe that the action of $G$ on $\varPhi$ ensures us that the arrow $( f, g , f \circ g )$ is well-defined for any $f \in \varPhi$ and any $g \in G$.

Now we can define a ``functorial'' version of the concept of GENEO.
\begin{definition}
Let us consider two categories $C(\varPhi, G)$ and $C(\Psi, H)$. A functor $F$ from $C(\varPhi, G)$ to $C(\Psi, H)$ is a $C$-GENEO if:
\begin{itemize}
    \item $D_\Psi(F(\varphi), F(\varphi')) \le D_\varPhi(\varphi, \varphi')$ for any $\varphi, \ \varphi' \in \varPhi$;
    \item for any pair of morphisms $m, \ m' \in \mathrm{hom}(C(\varPhi,G))$ such that $U_\R(m) = U_\R(m')$ we have that $U_\R(F(m)) = U_\R(F(m')).$
\end{itemize}
\end{definition}
GENEOs and $C$-GENEOs share the non-expansivity condition. The proposition below shows that the second conditions respectively required in the definitions of GENEO and $C$-GENEO correspond to each other in a suitable sense. We omit its trivial proof.
\begin{proposition}
Let $F$ be a functor from $C(\varPhi,G)$ to $C(\Psi,H)$. The following conditions are equivalent:
\begin{itemize}
    \item there exists a group homomorphism $T: G \longrightarrow H$ such that $F( \varphi \circ g) = F(\varphi) \circ T(g)$ for any $\varphi \in \varPhi$ and any $g \in G$;
    \item for any pair of morphisms $m, \ m' \in \mathrm{hom}(C(\varPhi,G))$ such that $U_\R(m) = U_\R(m')$ we have that $U_\R(F(m)) = U_\R(F(m'))$.
\end{itemize}
\end{proposition}

\section{Proofs}
\label{AppendixProofs}

\begin{theorem*}[\ref{tau=tau'}]
The topology  $\tau_{D_X}$ on $X$ induced by the pseudo-metric $D_X$ is finer than the initial topology $\tau_{\mathrm{in}}$ on  $X$ with respect to $\varPhi$.
If $\varPhi$ is totally bounded, then the topology  $\tau_{D_X}$ coincides with $\tau_{\mathrm{in}}$.
\end{theorem*}

\begin{proof}
We know that the set $\mathcal{B}_{D_X}=\left\{B_{X}(x, \varepsilon):x\in X,\varepsilon>0\right\}$ is a base for the topology $\tau_{D_X}$ and the set $\mathcal{B}_{\mathrm{in}}=\left\{\bigcap_{i \in I}\varphi^{-1}_i( U_i ):|I|<\infty, U_i \in \mathcal{T}_E \ \forall i\in I \right\}$ is a base for the topology $\tau_{\mathrm{in}}$.

First of all we have to show that the topology  $\tau_{D_X}$ is finer than the initial topology $\tau_{\mathrm{in}}$.
Let us take a set in the base $\mathcal{B}_{\mathrm{in}}$ of $\tau_{\mathrm{in}}$, i.e. a set $\bigcap_{i\in I}\varphi_i^{-1}(U_i)$, where $I$ is a finite set of indexes and $U_i\in \mathcal{T}_E$ for every index $i\in I$.
It will be sufficient to show that for every $y\in \bigcap_{i\in I}\varphi_i^{-1}(U_i)$ a ball $B_{X}(y,\varepsilon)\in \mathcal{B}_{D_X}$ exists, such that $B_{X}(y, \varepsilon)\subseteq\bigcap_{i \in I} \varphi_{i}^{-1}(U_i)$. Since $y\in \bigcap_{i\in I}\varphi_i^{-1}(U_i)$, we have that
$\varphi_i(y)\in U_i$ for every $i\in I$. Therefore, for each $i\in I$
we can find an open interval $]a_i,b_i[$ such that $\varphi_i(y)\in ]a_i,b_i[\subseteq U_i$. Let us set $\varepsilon:=\min_{i\in I}\min\{\varphi_i(y)-a_i,b_i-\varphi_i(y)\}$, and observe that $\varepsilon>0$.
If $z\in B_{X}(y,\varepsilon)$, then $|\varphi(y)-\varphi(z)|<\varepsilon$ for every $\varphi\in\varPhi$, and in particular
$|\varphi_i(y)-\varphi_i(z)|<\varepsilon$ for every $i\in I$. Hence
the definition of $\varepsilon$ immediately implies that $\varphi_i(z)\in]a_i,b_i[$ for every $i\in I$, so that $z\in \bigcap_{i \in I} \varphi_{i}^{-1}(]a_i,b_i[)$. It follows that $B_{X}(y, \varepsilon)\subseteq \bigcap_{i \in I} \varphi_{i}^{-1}(]a_i,b_i[)\subseteq \bigcap_{i \in I} \varphi_{i}^{-1}(U_i)$. Therefore, $y\in B_{X}(y, \varepsilon)
\subseteq
\bigcap_{i \in I} \varphi_{i}^{-1}(U_i)$, and our first statement is proved.

If $\varPhi$ is totally bounded, Proposition~\ref{stimaphi} in Appendix A guarantees that for every $\delta>0$
a finite subset $\varPhi_{\delta}$ of $\varPhi$ exists such that
\begin{equation}
\label{approxineq}
\left|\sup_{\varphi \in \varPhi}|\varphi(x_1)-\varphi(x_2)| - \max_{\varphi \in \varPhi_{\delta}}|\varphi(x_1) - \varphi(x_2)|\right| < 2\delta
\end{equation}
 for every $x_1, x_2\in X$.
Let us now set $B_{\delta}(x, r):= \left\{ x' \in X \Big| \max_{\varphi_i \in \varPhi_{\delta}}|\varphi_i(x)- \varphi_i(x')|< r\right\}$ for every $x\in X$ and every $r>0$.
We have to prove that the initial topology $\tau_{\mathrm{in}}$ is finer than the topology $\tau_{D_X}$.
In order to do this, it will be sufficient to show that for every $y\in B_{X}(x, \varepsilon)\in \mathcal{B}_{D_X}$ a set $\bigcap_{i \in I} \varphi_{i}^{-1}(U_i)\in\mathcal{B}_{\mathrm{in}}$ exists, such that $y\in \bigcap_{i \in I} \varphi_{i}^{-1}(U_i)\subseteq B_{X}(x, \varepsilon)$.

Let us choose a positive $\delta$ such that $2\delta< \varepsilon$.
Inequality~(\ref{approxineq}) implies that $B_{\delta}(y, \varepsilon-2\delta)\subseteq B_{X}(y,\varepsilon)$.
We now set $U_i:=]\varphi_i(y) - \varepsilon + 2\delta, \varphi_i(y)+ \varepsilon-2\delta [$ for $i\in I$.
Obviously, $y\in \bigcap_{\varphi_i \in \varPhi_{\delta}}\varphi_i^{-1} (U_i)$.
 If $z \in \bigcap_{\varphi_i \in \varPhi_{\delta}}\varphi_i^{-1} (U_i)$, then $|\varphi_i(z)-\varphi_i(y)|<\varepsilon-2\delta$ for every $\varphi_i \in \varPhi_{\delta}$. Hence, $z \in B_{\delta}(y, \varepsilon-2\delta)$. It follows that
 $\bigcap_{\varphi_i \in \varPhi_{\delta}}\varphi_i^{-1} (U_i)\subseteq B_{\delta}(y, \varepsilon-2\delta)$.
Therefore, $y\in \bigcap_{\varphi_i \in \varPhi_{\delta}}\varphi_i^{-1} (U_i) \subseteq B_{X}(x,\varepsilon)$ because of the inclusion $B_{\delta}(y, \varepsilon-2\delta)\subseteq B_{X}(y,\varepsilon)$.
This means that $\tau_{\mathrm{in}}$ is finer than $\tau_{D_X}$.
Since we already know that $\tau_{D_X}$ is finer than $\tau_{\mathrm{in}}$, it follows that
$\tau_{D_X}$ coincides with $\tau_{\mathrm{in}}$.
\end{proof}

\begin{remark*}
The second statement of Theorem~\ref{tau=tau'} becomes false if $\varPhi$ is not totally bounded. For example, assume $\varPhi$ equal to the set of all functions from $X=[0,1]$ to $\mathbb{R}$ that are continuous with respect to the Euclidean topologies on $[0,1]$ and $\mathbb{R}$. Indeed, it is easy to check that in this case $\tau_{D_X}$ is the discrete topology, while the initial topology $\tau_{\mathrm{in}}$ is the Euclidean topology on $[0,1]$.
\end{remark*}

\begin{remark*}
The pseudo-metric space $(X,D_X)$ may not be a $T_0$-space. For example, this happens if $X$ is a space containing at least two points and $\varPhi$ is the set of all the constant functions from $X$ to $[0,1]$.
\end{remark*}


\begin{theorem*}[\ref{Xcomplete}]
If $\varPhi$ is compact and $X$ is complete then $X$ is also compact.
\end{theorem*}

\begin{proof}
First of all we want to prove that every sequence $(x_i)$ in $X$ admits a Cauchy subsequence in $X$. After that, the statement follows immediately because every Cauchy sequence in a complete space is convergent, so that $X$ is sequentially compact, and hence compact, since $X$ is a pseudo-metric space~\cite{Ga64}.

Let us consider an arbitrary sequence $(x_i)$ in $X$ and an arbitrarily small $\varepsilon > 0$. Since $\varPhi$ is compact, we can find a finite subset $\varPhi_{\varepsilon} = \{ \varphi_1, \dots , \varphi_n \}$ such that $\varPhi = \bigcup_{i=1}^n B_{\varPhi}(\varphi_i, \varepsilon)$, where
$B_{\varPhi}(\varphi,\varepsilon)=\{\varphi' \in \varPhi : D_\varPhi(\varphi',\varphi) < \varepsilon \}$.
In particular, we can say that for any $\varphi \in \varPhi$ there exists $\varphi_{\bar k} \in \varPhi_\varepsilon$ such that $\|\varphi - \varphi_{\bar k}\|_{\infty} < \varepsilon$. Now, we consider the real sequence $\varphi_1(x_i)$ that is bounded because all the functions in $\varPhi$ are bounded. From Bolzano-Weierstrass Theorem it follows that we can extract a convergent subsequence $\varphi_1(x_{i_h})$. Then we consider the sequence $\varphi_2(x_{i_h})$. Since $\varphi_2$ is bounded, we can extract a convergent subsequence $\varphi_2(x_{i_{h_t}})$. We can repeat the same argument for any $\varphi_k \in \varPhi_{\varepsilon}$. Thus, we obtain a subsequence $(x_{i_{j}})$ of $(x_i)$, such that $\varphi_k(x_{i_j})$ is a real convergent sequence for any $k \in \{1, \dots, n \}$, and hence a Cauchy sequence in $\mathbb{R}$. Moreover, since $\varPhi_{\varepsilon}$ is a finite set, there exists an index $\bar{\jmath}$ such that for any $k \in \{1, \dots, n \}$ we have that

\begin{equation}
|\varphi_k(x_{i_r}) - \varphi_k(x_{i_s})| < \varepsilon, \ \ \forall \ r,s \geq \bar{\jmath}.
\end{equation}

We observe that $\bar{\jmath}$ does not depend on $\varphi$, but only on $\varepsilon$ and $\varPhi_{\varepsilon}$.

In order to prove that $(x_{i_{j}})$ is a Cauchy sequence in $X$, we observe that for any $r,s \in \mathbb{N}$ and any $\varphi \in \varPhi$ we have:

\begin{align}
|\varphi(x_{i_r}) - \varphi(x_{i_s})| &= |\varphi(x_{i_r}) - \varphi_k(x_{i_r}) + \varphi_k(x_{i_r}) - \varphi_k(x_{i_s}) + \varphi_k(x_{i_s}) - \varphi(x_{i_s})| \nonumber\\
& \leq  |\varphi(x_{i_r}) - \varphi_k(x_{i_r})| + |\varphi_k(x_{i_r}) - \varphi_k(x_{i_s})| + |\varphi_k(x_{i_s}) - \varphi(x_{i_s})|\\
& \leq  \|\varphi - \varphi_k\|_{\infty} + |\varphi_k(x_{i_r}) - \varphi_k(x_{i_s})| + \|\varphi_k - \varphi\|_{\infty}.  \nonumber
\end{align}

It follows that $|\varphi(x_{i_r}) - \varphi(x_{i_s})| < 3\varepsilon$ for every $\varphi \in \varPhi$ and every $r,s \geq \bar{\jmath}$.
Thus, $\sup_{\varphi \in \varPhi}|\varphi(x_{i_r}) - \varphi(x_{i_s})| = D_X(x_{i_r},x_{i_s}) \le 3\varepsilon$. Hence, the sequence $(x_{i_j})$ is a Cauchy sequence in $X$. The completeness of $X$ implies that the statement of Theorem~\ref{Xcomplete} is true.
\end{proof}

\begin{example*}
Let $\varPhi$ be the set containing all the $1$-Lipschitz functions from $X=\{(x,y)\in\R^3:x^2+y^2=1,\arcsin(x)\in \mathbb{Q}\}$ to $[0,1]$, and $G$ be the group of all rotations $\rho_{2\pi q}$ of $2\pi q$ radians with $q\in \mathbb{Q}$. The topological space $X$ is neither complete nor compact.
\end{example*}


\begin{proposition*}[\ref{propgisometry}]
If $g$ is a bijection from $X$ to $X$ such that $\varphi\circ g\in\varPhi$ and $\varphi\circ g^{-1}\in\varPhi$ for every $\varphi\in\varPhi$, then $g$ is an isometry (and hence a homeomorphism) with respect to $D_X$.
\end{proposition*}

\begin{proof}
Let us fix two arbitrary points $x,x'$ in $X$. Obviously, the map $R_g:\varPhi\to\varPhi$ taking each function $\varphi$ to $\varphi\circ g$ is surjective, since $\varphi=R_g\left(R_{g^{-1}}(\varphi)\right)$.  Hence $R_g(\varPhi)=\varPhi$. Therefore, $g$ preserves the pseudo-distance $D_X$:
\begin{align}
D_X(g(x),g(x')) &= \sup_{\varphi \in \varPhi}|\varphi(g(x))-\varphi(g(x'))| \nonumber \\
&= \sup_{\varphi \in \varPhi}|(\varphi \circ g)(x)-(\varphi \circ g)(x')|\\
&= \sup_{\varphi \in R_g(\varPhi)}|\varphi(x)-\varphi(x')|\\
&= \sup_{\varphi\in \varPhi}|\varphi(x)-\varphi(x')|= D_X(x,x'). \nonumber
\end{align}
Since $g$ is bijective, it follows that $g$ is an isometry with respect to $D_X$.
\end{proof}


\begin{theorem*}[\ref{thmRgcontinuous}]
$G$ is a topological group with respect to the pseudo-metric topology and the action of $G$ on $\varPhi$ through right composition is continuous.
\end{theorem*}

\begin{proof}
It will suffice to prove that if $f=\lim_{i \to +\infty} f_i$ and $g=\lim_{i \to +\infty} g_i$ in $G$ with respect to the pseudo-metric $D_G$,
then $g  \circ f=\lim_{i \to +\infty} g_i \circ f_i$ and $f^{-1}=\lim_{i \to +\infty} f_i^{-1}$.

Because of the compactness of $\varPhi$ and Proposition~\ref{stimaphi}, for every $\delta>0$ we can take a finite subset $\varPhi_{\delta}$ of $\varPhi$ such that
\begin{equation*}
    \left|\sup_{\varphi \in \varPhi}|\varphi(x_1)-\varphi(x_2)| - \max_{\varphi \in \varPhi_{\delta}}|\varphi(x_1) - \varphi(x_2)|\right| < 2\delta
\end{equation*}
for every $x_1, x_2\in X$.
We have that
\begin{align}
&D_G(g_i \circ f_i,g \circ f)  \le D_G(g_i \circ f_i,g \circ f_i)+  D_G(g \circ f_i,g \circ f)=\nonumber
 \\& = \sup_{\varphi \in \varPhi}\left \|\varphi \circ( g_i \circ f_i) - \varphi \circ (g \circ f_i)  \right \|_{\infty} + \sup_{\varphi \in \varPhi} \left \|\varphi \circ( g \circ f_i) - \varphi \circ (g \circ f)  \right \|_{\infty}
 \\& = \sup_{\varphi \in \varPhi}\ \sup_{x \in X} | \varphi(g_i(f_i(x))- \varphi(g(f_i(x))| + \sup_{\varphi \in \varPhi}\ \sup_{x \in X} | \varphi(g(f_i(x))- \varphi(g(f(x))|\nonumber
 \\&= \sup_{\varphi \in \varPhi}\ \sup_{y \in X} | \varphi (g_i (y))- \varphi (g(y))| + \sup_{\varphi \in \varPhi}\ \sup_{x \in X} | \varphi (g(f_i(x))- \varphi (g(f(x))|\nonumber
\\ & < D_G(g_i,g) + \max_{\varphi \in \varPhi_\delta }\ \sup_{x \in X} | \varphi (g(f_i(x))- \varphi (g(f(x))| + 2\delta.
\end{align}
Since $g=\lim_{i \to +\infty} g_i$,
$\lim_{i \to +\infty} D_G(g_i,g)=0$.
Because of Theorem~\ref{Xcomplete}, $X$ is compact and hence $\varphi\circ g:X\to \R$ is a uniformly continuous function.
Since $f=\lim_{i \to +\infty} f_i$,
it follows that $\lim_{i \to +\infty} \sup_{x \in X} | \varphi (g(f_i(x))- \varphi (g(f(x))|=0$ for every $\varphi\in\varPhi_\delta$, and hence
$\lim_{i \to +\infty} \max_{\varphi \in \varPhi_\delta}\ \sup_{x \in X} | \varphi (g(f_i(x))- \varphi (g(f(x))|=0$.
Given that $\delta$ can be taken arbitrarily small, we get $g  \circ f=\lim_{i \to +\infty} g_i \circ f_i$.

We also want to prove that $f^{-1}=\lim_{i \to +\infty} f_i^{-1}$.
By contradiction, if we had not that $\lim_{i \to \infty} D_G(f_i^{-1},f^{-1})=0$, then there would exist a constant $c>0$ and a subsequence $(f_{i_j})$ of $(f_i)$
such that $D_G(f_{i_j}^{-1},f^{-1}) \geq c >0$ for every index $j$. However, we should still have $\lim_{j \to \infty} D_G(f_{i_j},f)=0$ because $(f_{i_j})$ is a subsequence of $(f_i)$.
Since $D_G(f_{i_j}^{-1},f^{-1}) \geq c >0$ for every index $j$, a $\varphi_j \in \varPhi$ should exist such that $\|\varphi_j \circ f_{i_j}^{-1}- \varphi_j \circ f^{-1}\|_{\infty} \geq c >0$.

Because of the compactness of $\varPhi$, it would not be restrictive to assume (possibly by considering subsequences) the existence of the following limits: $\bar{\varphi}= \lim_{j \to \infty} \varphi_j$ and $\hat{\varphi}=\lim_{j \to \infty} \varphi_j \circ f_{i_j}^{-1}$. We would have that
\begin{align}
 D_\varPhi(\hat{\varphi},\bar{\varphi} \circ f^{-1}) & =D_\varPhi(\lim_{j \to \infty} \varphi_j \circ f_{i_j}^{-1},\lim_{j \to \infty} \varphi_j \circ f^{-1})\nonumber\\
& = \lim_{j \to \infty} D_\varPhi(\varphi_j \circ f_{i_j}^{-1},\varphi_j \circ f^{-1})\geq c > 0
\end{align}
so that $\hat{\varphi}\neq \bar{\varphi} \circ f^{-1}$.\\
On the other hand, we should have
\begin{align}
 D_\varPhi(\hat{\varphi} \circ f,\bar{\varphi}) & =D_\varPhi((\lim_{j \to \infty} \varphi_j \circ f_{i_j}^{-1}) \circ f,\lim_{j \to \infty} \varphi_j)\nonumber\\
& = \lim_{j \to \infty} D_\varPhi((\varphi_j \circ f_{i_j}^{-1}) \circ f,(\varphi_j \circ f_{i_j}^{-1})\circ f_{i_j})\\
& \le \lim_{j \to \infty} D_G(f_{i_j},f)=0\nonumber
\end{align}
so that $\hat{\varphi} \circ f= \bar{\varphi}$.

It follows that $R_f$ is not injective, against our assumptions.

This contradiction proves that $\lim_{i \to \infty}f_i^{-1}= f^{-1}$.

Therefore, $G$ is a topological group.

Let now $\varepsilon$ be a positive real number. If $D_\varPhi(\varphi,\psi),D_G(f,g) < \delta:=\varepsilon/2$ then

\begin{align}
D_\varPhi(\varphi\circ f,\psi\circ g)
& \le D_\varPhi(\varphi\circ f,\varphi\circ g) +D_\varPhi(\varphi\circ g,\psi\circ g)\nonumber\\
& = D_\varPhi(\varphi\circ f,\varphi\circ g) +D_\varPhi(\varphi,\psi)\\
& \le D_G(f,g)+D_\varPhi(\varphi,\psi) < \varepsilon/2+\varepsilon/2=\varepsilon.\nonumber
\end{align}
This proves that the action of $G$ on $\varPhi$ through right composition is continuous.
\end{proof}


\begin{theorem*}[\ref{thmGcompact}]
If $G$ is complete then it is also compact with respect to $D_G$.
\end{theorem*}

\begin{proof}
We want to show that $G$ is sequentially compact, and hence compact.
Let $(g_i)$ be a sequence in $G$ and take a real number $\varepsilon > 0$. Given that $\varPhi$ is compact, we can find a finite subset $\varPhi_{\varepsilon} = \{ \varphi_1, \dots , \varphi_n \}$ such that for every $\varphi\in\varPhi$ there exists $\varphi_h \in \varPhi_\varepsilon$ for which $D_\varPhi(\varphi_h,\varphi) < \varepsilon$. For any fixed $k \in \{1, \dots , n\}$, let us consider the sequence $(\varphi_k \circ g_i)$ in $\varPhi$. Applying the same argument as in the proof of Theorem~\ref{Xcomplete}, we can extract a subsequence $(g_{i_j})$ of $(g_i)$ such that $(\varphi_k \circ g_{i_j})$ converges in $\varPhi$ with respect to $D_\varPhi$ and hence it is a Cauchy sequence for any $k \in \{1, \dots , n\}$. For the finiteness of set $\varPhi_\varepsilon$, we can find an index $\bar{\jmath}$ such that

\begin{equation}
D_\varPhi(\varphi_k \circ g_{i_r},\varphi_k \circ g_{i_s}) < \varepsilon, \ \textnormal{for every} \ s,r \geq \bar{\jmath}.
\end{equation}

In order to prove that $(g_{i_{j}})$ is a Cauchy sequence, we observe that for any $\varphi\in\varPhi$, any $\varphi_k\in\varPhi_\varepsilon$, and any $r,s \in \mathbb{N}$ we have
\begin{align}
& D_\varPhi(\varphi \circ g_{i_r},\varphi \circ g_{i_s})) \nonumber \\
& \leq D_\varPhi(\varphi \circ g_{i_r},\varphi_k \circ g_{i_r}) + D_\varPhi(\varphi_k \circ g_{i_r},\varphi_k \circ g_{i_s}) + D_\varPhi(\varphi_k \circ g_{i_s},\varphi \circ g_{i_s}) \\
& =  D_\varPhi(\varphi,\varphi_k) + D_\varPhi(\varphi_k \circ g_{i_r},\varphi_k \circ g_{i_s}) + D_\varPhi(\varphi_k, \varphi).\nonumber
\end{align}

We observe that $\bar{\jmath}$ does not depend on $\varphi$, but only on $\varepsilon$ and $\varPhi_{\varepsilon}$. By choosing a $\varphi_k\in\varPhi_\varepsilon$ such that $D_\varPhi(\varphi_k,\varphi) < \varepsilon$, we get $D_\varPhi(\varphi \circ g_{i_r},\varphi \circ g_{i_s})) < \varepsilon$ for every $\varphi \in \varPhi$ and every $r,s \geq \bar{\jmath}$.
Thus, $D_G(g_{i_r},g_{i_s}) < 3\varepsilon$. Hence, the sequence $(g_{i_j})$ is a Cauchy sequence. Finally, given that $G$ is complete, $(g_{i_j})$ is convergent. Therefore, $G$ is sequentially compact.
\end{proof}

\begin{example*}
Let $\varPhi$ be the set containing all the $1$-Lipschitz functions from $X=S^1=\{(x,y)\in\R^3:x^2+y^2=1\}$ to $[0,1]$, and $G$ be the group of all rotations $\rho_{2\pi q}$ of $X$ of $2\pi q$ radians with $q$ rational number. The space $(G,D_G)$ is neither complete nor compact.
\end{example*}


\begin{proposition*}[\ref{thmContraction}]
If $F$ is a GENEO from $(\varPhi, G)$ to $(\varPsi, H)$ associated with $T:G\to H$, then it is a contraction with respect to the natural pseudo-distances $d_G$, $d_H$.
\end{proposition*}

\begin{proof}%

Since $F$ is a GENEO, it follows that
\begin{align}
d_H(F(\varphi_1), F(\varphi_2)) & =\inf_{h \in H}D_\Psi\left(F(\varphi_1),F(\varphi_2) \circ h\right)\nonumber\\
& \leq \inf_{g \in G} D_\Psi\left( F(\varphi_1), F(\varphi_2) \circ T(g)\right)\\
& = \inf_{g \in G} D_\Psi\left(F(\varphi_1),F(\varphi_2 \circ g)\right)\nonumber\\
& \leq \inf_{g \in G} D_\varPhi\left(\varphi_1,\varphi_2 \circ g\right) = d_G(\varphi_1, \varphi_2).\nonumber
\end{align}
\end{proof}


\begin{proposition*}[\ref{propzerof}]
For every $F\in \mathcal{F}^{\mathrm{all}}$ and every $\varphi \in \varPhi$:
$\|F(\varphi)\|_{\infty}\le \|\varphi\|_{\infty}+ \|F(\textbf{0})\|_{\infty}$, where $\textbf{0}$ denotes the function taking the value 0 everywhere.
\end{proposition*}

\begin{proof}
Since $F$ is non-expansive, we have that
\begin{align*}
    \|F(\varphi)\|_{\infty} &= \|F(\varphi) - F(\textbf{0}) + F(\textbf{0})\|_{\infty}\\
    & \le \|F(\varphi) - F(\textbf{0})\|_{\infty} + \|F(\textbf{0})\|_{\infty}\\
    & \le \|\varphi - \textbf{0}\|_{\infty} + \|F(\textbf{0})\|_{\infty}= \|\varphi\|_{\infty} + \|F(\textbf{0})\|_{\infty}.
\end{align*}
\end{proof}


\begin{theorem*}[\ref{t17}]
$\mathcal{F}^{\mathrm{all}}$ is compact with respect to $D_{\GENEO}$.
\end{theorem*}

\begin{proof}
We know that $(\mathcal{F}^{\mathrm{all}},D_{\GENEO})$ is a metric space. Therefore it will suffice to prove that $\mathcal{F}^{\mathrm{all}}$ is sequentially compact.
In order to do this, let us assume that a sequence $(F_i)$ in $\mathcal{F}^{\mathrm{all}}$ is given.
Given that $\varPhi$ is a compact (and hence separable) metric space, we can find a countable and dense subset $\varPhi^*=\{\varphi_j \}_{j\in \mathbb{N}}$ of $\varPhi$.
By means of a diagonalization process, we can extract a subsequence $(F'_{i})$ from $(F_i)$, such that for every fixed index $j$ the sequence $(F'_{i}(\varphi_j))$ converges to a function in $\varPsi$ with respect to $D_\Psi$.
%
Now, let us consider the function $\bar{F}: \varPhi \rightarrow\ \varPsi$ defined by setting
$\bar{F}(\varphi_j):= \lim_{i \to \infty}F'_{i}(\varphi_j) $ for each $\varphi_j \in \varPhi^{*}$.

We extend $\bar{F}$ to $\varPhi$ as follows. For every $\varphi \in \varPhi$ we choose a sequence $(\varphi_{j_r})$ in $\varPhi^{*}$, converging to $\varphi \in \varPhi$, and set $\bar{F}(\varphi):= \lim_{r \to \infty} \bar{F}(\varphi_{j_r})$. We claim that such a limit exists in $\varPsi$ and does not depend on the sequence that we have chosen, converging to $\varphi \in \varPhi$. In order to prove that the previous limit exists, we observe that for every $r, \ s \in \mathbb{N}$
\begin{align*}
 D_\Psi\left(\bar{F}(\varphi_{j_r}),\bar{F}(\varphi_{j_s})\right) & = D_\Psi\left(\lim_{i \to \infty}F'_{i}(\varphi_{j_r}),\lim_{i \to \infty}F'_{i}(\varphi_{j_s})\right) \\
 & = \lim_{i \to \infty} D_\Psi\left(F'_{i}(\varphi_{j_r}),F'_{i}(\varphi_{j_s})\right) \\
 & \le \lim_{i \to \infty} D_\varPhi\left(\varphi_{j_r},\varphi_{j_s}\right)= D_\varPhi\left(\varphi_{j_r},\varphi_{j_s}\right),
\end{align*}
because each $F'_{i}$ is non-expansive.

Since the sequence $(\varphi_{j_r})$ converges to $\varphi \in \varPhi$, it follows that $(\bar{F}(\varphi_{j_r}))$ is a Cauchy sequence with respect to $D_\Psi$. The compactness of $\varPsi$ implies that $(\bar{F}(\varphi_{j_r}))$ converges in $\varPsi$.\\
If another sequence $(\varphi_{k_r})$ in given in $\varPhi^{*}$, converging to $\varphi \in \varPhi$, then for every index $r \in \mathbb{N}$
\begin{align*}
 D_\Psi\left(\bar{F}(\varphi_{j_r}),\bar{F}(\varphi_{k_r})\right) & = D_\Psi\left(\lim_{i \to \infty}F'_{i}(\varphi_{j_r}),\lim_{i \to \infty}F'_{i}(\varphi_{k_r})\right) \\
 & = \lim_{i \to \infty} D_\Psi\left(F'_{i}(\varphi_{j_r}),F'_{i}(\varphi_{k_r})\right) \\
 & \le \lim_{i \to \infty} D_\varPhi\left(\varphi_{j_r},\varphi_{k_r}\right) \\
 & = D_\varPhi\left(\varphi_{j_r},\varphi_{k_r}\right).
\end{align*}

Since both $(\varphi_{j_r})$ and $(\varphi_{k_r})$ converge to $\varphi$ it follows that
$\lim_{r \to \infty} \bar{F}(\varphi_{j_r})=\lim_{r \to \infty} \bar{F}(\varphi_{k_r})$. Therefore the definition of $\bar{F}(\varphi)$ does not depend on the sequence $(\varphi_{j_r})$ that we have chosen, converging to $\varphi$.

Now we have to prove that $\bar{F} \in \mathcal{F}^{\mathrm{all}}$, i.e., that $\bar{F}$ verifies the properties defining this set of operators.
We have already seen that $\bar{F}:\varPhi \rightarrow\ \varPsi$.

For every $\varphi, \varphi'$ we can consider two sequences $(\varphi_{j_r})$, $(\varphi_{k_r})$ in $\varPhi^{*}$, converging to $\varphi$ and $\varphi'$, respectively. Due to the fact that the operators $F'_{i}$ are non-expansive, we have that
\begin{align*}
 D_\Psi\left(\bar{F}(\varphi),\bar{F}(\varphi')\right) & = D_\Psi\left(\lim_{r \to \infty}\bar{F}(\varphi_{j_r}),\lim_{r \to \infty}\bar{F}(\varphi_{k_r})\right) \\
& = D_\Psi\left(\lim_{r \to \infty}\lim_{i \to \infty}F'_{i}(\varphi_{j_r}),\lim_{r \to \infty} \lim_{i \to \infty}F'_{i}(\varphi_{k_r})\right) \\
 & =  \lim_{r \to \infty} \lim_{i \to \infty} D_\Psi\left(F'_{i}(\varphi_{j_r}),F'_{i}(\varphi_{k_r})\right) \\
 & \le \lim_{r \to \infty} \lim_{i \to \infty} D_\varPhi\left(\varphi_{j_r},\varphi_{k_r}\right) \\
  & = \lim_{r \to \infty} D_\varPhi\left(\varphi_{j_r},\varphi_{k_r}\right) \\
 & = D_\varPhi\left(\varphi,\varphi'\right).
\end{align*}
Therefore, $\bar{F}:\varPhi\to\Psi$ is non-expansive. As a consequence, it is also continuous.

We can now prove that the sequence $(F'_{i})$ converges to $\bar{F}$ with respect to $D_{\GENEO}$.

Let us consider an arbitrarily small $\varepsilon>0$. Since $\varPhi$ is compact and $\varPhi^{*}$ is dense in $\varPhi$, we can find a finite subset $\{ \varphi_{j_1},\dots, \varphi_{j_n} \}$ of $\varPhi^{*}$  such that for every $\varphi \in \varPhi$, there exists an index $r \in \{1, \dots, n \}$, for which $D_\varPhi\left(\varphi,\varphi_{j_r} \right) < \varepsilon$.

Since the sequence  $(F'_{i})$ converges pointwise to $\bar{F}$ on the set $\varPhi^{*}$, an index $\bar\imath$ exists, such that $D_\Psi\left(\bar{F}(\varphi_{j_r}),F'_{i}(\varphi_{j_r})\right) < \varepsilon$ for any $i \ge \bar\imath$ and any $r \in \{ 1, \dots , n \}$.
Therefore, for every $\varphi \in \varPhi$ we can find an index $r \in \{ 1, \dots, n \}$ such that $D_\varPhi\left(\varphi,\varphi_{j_r} \right) < \varepsilon$ and the following inequalities hold for every index $i\ge \bar\imath$, because of the non-expansivity of $\bar{F}$ and $F'_{i}$:
\begin{align*}
    D_\Psi\left(\bar{F}(\varphi),F'_{i}(\varphi)\right)& \le D_\Psi\left(\bar{F}(\varphi),\bar{F}(\varphi_{j_r})\right) +D_\Psi\left(\bar{F}(\varphi_{j_r}),F'_{i}(\varphi_{j_r})\right) +D_\Psi\left(F'_{i}(\varphi_{j_r}),F'_{i}(\varphi)\right) \\
    & \le D_\varPhi\left(\varphi,\varphi_{j_r}\right) + D_\Psi\left(\bar{F}(\varphi_{j_r}),F'_{i}(\varphi_{j_r})\right) + D_\varPhi\left(\varphi_{j_r},\varphi\right) < 3 \varepsilon.
\end{align*}

We observe that $\bar\imath$ does not depend on $\varphi$, but only on $\varepsilon$ and on the set $\{ \varphi_{j_1},\dots, \varphi_{j_n} \}$. It follows that  $D_\Psi\left(\bar{F}(\varphi),F'_{i}(\varphi)\right) < 3\varepsilon$ for every $\varphi \in \varPhi$ and every $i \ge \bar\imath$.

Hence, $\sup_{\varphi \in \varPhi}D_\Psi\left(\bar{F}(\varphi),F'_{i}(\varphi)\right)\le 3\varepsilon$ for every $i \ge \bar\imath$.
Therefore, the sequence $(F'_{i})$ converges to $\bar{F}$ with respect to $D_{\GENEO}$.

The last thing that we have to show is that $\bar{F}$ is group equivariant. Let us consider a $\varphi \in \varPhi$, a sequence $(\varphi_{j_r})$ in $\varPhi^{*}$ converging to $\varphi$ in $\varPhi$ and a $g \in G$.
Obviously, $D_\varPhi(\varphi_{j_r} \circ g ,\varphi \circ g)=D_\varPhi(\varphi_{j_r},\varphi)$ and hence the sequence $(\varphi_{j_r} \circ g )$ converges to $\varphi \circ g$ in $\varPhi$ with respect to $D_\varPhi$. We recall that the right action of $G$ on $\varPhi$ is continuous, $\bar{F}$ is continuous and each $F'_{i}$ is group equivariant. Hence, given that the sequence $(F'_{i})$ converges to $\bar{F}$ with respect to $D_{\GENEO}$, the following equalities hold:
\begin{align*}
    \bar{F}(\varphi \circ g) & = \bar{F}(\lim_{r \to \infty} (\varphi_{j_r} \circ g))\\
    & = \lim_{r \to \infty} \bar{F}(\varphi_{j_r} \circ g)\\
    & = \lim_{r \to \infty} \lim_{i \to \infty}F'_{i}(\varphi_{j_r} \circ g)\\
    & = \lim_{r \to \infty} \lim_{i \to \infty}  F'_{i}(\varphi_{j_r}) \circ T(g)\\
    & = \lim_{r \to \infty} \bar{F}(\varphi_{j_r}) \circ T(g)\\
    & = \bar{F}(\varphi) \circ T(g).
\end{align*}
This proves that $\bar{F}$ is group equivariant, and hence a perception map.
In conclusion, $\bar{F}$ is a GENEO.
From the fact that the sequence $F'_{i}$ converges to $\bar{F}$ with respect to $D_{\GENEO}$, it follows that $(\mathcal{F}^{\mathrm{all}},D_{\GENEO})$ is sequentially compact.
\end{proof}


\begin{proposition*}[\ref{propconvex}]
If $F_{\Sigma}(\varPhi)\subseteq \varPsi$, then $F_{\Sigma}$ is a GENEO from $(\varPhi, G)$ to $(\varPsi, H)$ with respect to $T$.
\end{proposition*}

\begin{proof}
First we prove that $F_{\Sigma}$ is a perception map with respect to $T$. Since every $F_i$ is a perception map we have that:
\begin{equation}
F_{\Sigma}(\varphi \circ g) = \sum_{i=1}^n a_i F_i (\varphi \circ g)=\sum_{i=1}^n a_i (F_i (\varphi) \circ T(g))= \sum_{i=1}^n (a_i F_i (\varphi)) \circ T(g)= F_{\Sigma}(\varphi) \circ T(g).
\end{equation}
Since every $F_i$ is non-expansive, $F_{\Sigma}$ is non-expansive:
\begin{align}
D_\Psi\left( F_{\Sigma}(\varphi_1),F_{\Sigma}(\varphi_2)\right) & = \left\|\sum_{i=1}^n a_i F_i(\varphi_1) - \sum_{i=1}^n a_i F_i(\varphi_2) \right\|_{\infty}\\
 & =  \left\|\sum_{i=1}^n a_i (F_i(\varphi_1) - F_i(\varphi_2))\right\|_{\infty} \\
 & \le  \sum_{i=1}^n |a_i| \left\|(F_i(\varphi_1) - F_i(\varphi_2))\right\|_{\infty} \\
 & \le  \sum_{i=1}^n |a_i| \left\|\varphi_1 - \varphi_2 \right\|_{\infty}\le D_\varPhi\left(\varphi_1,\varphi_2\right).
\end{align}
Therefore $F_{\Sigma}$ is a GENEO.
\end{proof}


\begin{theorem*}[\ref{thmconvex}]
If $\Psi$ is convex, then the set of GENEOs from $(\varPhi, G)$ to $(\varPsi, H)$ with respect to $T$ is convex.
\end{theorem*}

\begin{proof}
It is sufficient to apply Proposition~\ref{propconvex} for $n=2$, by setting $a_1=t$, $a_2=1-t$ for $0\le t\le 1$, and observing that the convexity of $\Psi$ implies $F_{\Sigma}(\varPhi)\subseteq \varPsi$.
\end{proof}

\begin{proposition*}
  $\mathcal{D}^{\mathcal{F},k}_{\mathrm{match}}$ is a strongly $G$-invariant pseudo-metric on $\varPhi$.
\end{proposition*}

\begin{proof}
Theorem~\ref{t12} and the non-expansivity of every $F \in \mathcal{F}$ imply that
\begin{align*}
   d_{\mathrm{match}} (r_k(F(\varphi_1)),r_k(F(\varphi_2)))& \le D_\Psi\left(F(\varphi_1),F(\varphi_2)\right)\\
   & \le D_\varPhi\left(\varphi_1,\varphi_2\right).
\end{align*}
Therefore $\mathcal{D}^{\mathcal{F},k}_{\mathrm{match}}$ is a pseudo-metric, since it is the supremum of a family of pseudo-metrics that are bounded at each pair $(\varphi_1,\varphi_2)$. Moreover, for every $\varphi_1,\varphi_2 \in \varPhi$ and every $g \in G$
\begin{align*}
  \mathcal{D}^{\mathcal{F},k}_{\mathrm{match}}(\varphi_1,\varphi_2\circ g)
  & := \sup_{F \in \mathcal{F}} d_{\mathrm{match}}(r_k(F(\varphi_1)),r_k(F(\varphi_2 \circ g)))\\
  & = \sup_{F \in \mathcal{F}} d_{\mathrm{match}}(r_k(F(\varphi_1)),r_k(F(\varphi_2) \circ T( g)))\\
  & = \sup_{F \in \mathcal{F}} d_{\mathrm{match}}(r_k(F(\varphi_1)),r_k(F(\varphi_2))\\
  & =\mathcal{D}^{\mathcal{F},k}_{\mathrm{match}}(\varphi_1,\varphi_2)\\
\end{align*}
because of the equality $F(\varphi \circ g)=F(\varphi) \circ T(g)$ for every $\varphi \in \varPhi$ and every $g\in G$ and the invariance of persistent homology under the action of the homeomorphisms. Since the function $\mathcal{D}^{\mathcal{F},k}_{\mathrm{match}}$ is symmetric, this is sufficient to guarantee that $\mathcal{D}^{\mathcal{F},k}_{\mathrm{match}}$ is strongly $G$-invariant.
\end{proof}


\begin{theorem*}[\ref{t15}]
If $\mathcal{F}$ is a non-empty subset of $\mathcal{F}^{\mathrm{all}}$, then
\begin{equation}\mathcal{D}^{\mathcal{F},k}_{\mathrm{match}} \le d_G \le D_\varPhi.
\end{equation}
\end{theorem*}

\begin{proof}
For every $F \in \mathcal{D}^{\mathcal{F},k}_{\mathrm{match}}$, every $g \in G$ and every $\varphi_1, \varphi_2 \in \varPhi$, we have that
\begin{align*}
    d_{\mathrm{match}}(r_k(F(\varphi_1)),r_k(F(\varphi_2))) & =d_{\mathrm{match}}(r_k(F(\varphi_1)),r_k(F(\varphi_2) \circ T( g)))\\
    & =d_{\mathrm{match}}(r_k(F(\varphi_1)),r_k(F(\varphi_2 \circ g)))\\
    & \le D_\Psi\left(F(\varphi_1),F(\varphi_2 \circ g)\right)\le D_\varPhi\left(\varphi_1,\varphi_2 \circ g\right).
\end{align*}
The first equality follows from the invariance of persistent homology under action of $\Homeo(X)$ (see Remark~\ref{r10}), and the second equality follows from the fact F is a group equivariant operator. The first inequality follows from the stability of persistent homology (Theorem~\ref{t12}), while the second inequality follows from the non-expansivity of $F$.
It follows that, if $\mathcal{F}\subseteq \mathcal{F}^{\mathrm{all}}$, then for every $g \in G$ and every $\varphi_1, \varphi_2 \in \varPhi$
\begin{equation}
\mathcal{D}^{\mathcal{F},k}_{\mathrm{match}}(\varphi_1,\varphi_2)\le D_\varPhi\left(\varphi_1,\varphi_2 \circ g\right).
\end{equation}
Hence, the inequality $\mathcal{D}^{\mathcal{F},k}_{\mathrm{match}} \le d_G$ follows, while $d_G \le D_\varPhi$ is stated in Theorem~\ref{t12}.
\end{proof}


\begin{theorem*}[\ref{thrEqualityDFall_dmatch}]
Let us assume that $\varPhi=\Psi$, every function in $\varPhi$ is non-negative, the $k$-th Betti number of $X$ does not vanish, and $\varPhi$ contains each constant function $c$ for which a function $\varphi \in \varPhi $ exists such that $0\le c \le \|\varphi\|_{\infty}$.
Then $\mathcal{D}^{\mathcal{F}^{\mathrm{all}},k}_{\mathrm{match}}=d_G$.
\end{theorem*}

\begin{proof}
For every $\varphi' \in \varPhi$ let us consider the operator $F_{\varphi'}: \varPhi \to \varPhi$ defined by setting $F_{\varphi'}(\varphi)$ equal to the constant function taking everywhere the value $d_G(\varphi, \varphi')$ for every $\varphi \in \varPhi$ (i.e., $F_{\varphi'}(\varphi)(x)=d_G(\varphi,\varphi')$ for any $x \in X$). Our assumptions guarantee that such a constant function belongs to $\varPhi=\Psi$. We also set $T=\mathrm{id}:G\to G$.

We observe that
\begin{enumerate}
    \item $F_{\varphi'}$ is a group equivariant operator on $\varPhi$, because the strong invariance of the natural pseudo-distance $d_G$ with respect to the group $G$ (Remark~\ref{invdg}) implies that if $\varphi \in \varPhi$ and $g \in G$, then $F_{\varphi'}(\varphi \circ g)(x) = d_G(\varphi\circ g,\varphi') = F_{\varphi'}(\varphi)(g(x)) = (F_{\varphi'}(\varphi)\circ g)(x) = (F_{\varphi'}(\varphi)\circ T(g))(x)$, for every $x\in X$.
    \item $F_{\varphi'}$ is non-expansive on $\varPhi$, because for every $\varphi_1, \varphi_2 \in \varPhi$
        \begin{align*}
D_\Psi\left(F_{\varphi'}(\varphi_1),F_{\varphi'}(\varphi_2)\right) & = |d_G(\varphi_1, \varphi') - d_G(\varphi_2,\varphi')|\\
& \le d_G(\varphi_1,\varphi_2) \le D_\varPhi\left(\varphi_1,\varphi_2 \right).
\end{align*}
\end{enumerate}
Therefore, $F_{\varphi'}$ is a GENEO.

For every $\varphi_1,\varphi_2,\varphi' \in \varPhi$ we have that
\begin{equation}d_{\mathrm{match}}(r_k(F_{\varphi'}(\varphi_1)),r_k(F_{\varphi'}(\varphi_2)))=|d_G(\varphi_1, \varphi') - d_G(\varphi_2,\varphi')|.
\end{equation}
Indeed, apart from the trivial points on the line $\{(u,v) \in \mathbb{R}^2 \ : \ u=v \}$, the persistence diagram associated with $r_k(F_{\varphi'}(\varphi_1))$ contains only the point $(d_G(\varphi_1,\varphi'),\infty)$, while the persistence diagram associated with $r_k(F_{\varphi'}(\varphi_2))$ contains only the point $(d_G(\varphi_2,\varphi'),\infty)$. Both the points have the same multiplicity, which equals the (non-null) $k$-th Betti number of $X$.

Setting $\varphi'=\varphi_2$, we have that
\begin{equation}
d_{\mathrm{match}}(r_k(F_{\varphi'}(\varphi_1)),r_k(F_{\varphi'}(\varphi_2)))=d_G(\varphi_1,\varphi_2).
\end{equation}
As a consequence, we have that \begin{equation}
\mathcal{D}^{\mathcal{F}^{\mathrm{all}},k}_{\mathrm{match}}(\varphi_1,\varphi_2)\ge d_G(\varphi_1,\varphi_2).
\end{equation}
By applying Theorem~\ref{t15}, we get
\begin{equation}
\mathcal{D}^{\mathcal{F}^{\mathrm{all}},k}_{\mathrm{match}}(\varphi_1,\varphi_2)= d_G(\varphi_1,\varphi_2)
\end{equation}
for every $\varphi_1,\varphi_2$.
\end{proof}


\begin{proposition*}[\ref{14}]
Let $\mathcal{F},\mathcal{F}'\subseteq \mathcal{F}^\mathrm{all}$.
If the Hausdorff distance $$HD(\mathcal{F},\mathcal{F}'):=
\max\left\{
\sup_{F\in\mathcal{F}}\inf_{F'\in\mathcal{F}'} D_{\GENEOH}(F,F'),
\sup_{F'\in\mathcal{F}'}\inf_{F\in\mathcal{F}} D_{\GENEOH}(F,F')\right\}$$  is not larger than $\varepsilon$, then
\begin{equation}\left|\mathcal{D}^{\mathcal{F},k}_{\mathrm{match}}(\varphi_1, \varphi_2) - \mathcal{D}^{\mathcal{F}',k}_{\mathrm{match}}(\varphi_1, \varphi_2)\right|\le 2 \varepsilon
\end{equation}
for every $\varphi_1,\varphi_2 \in \varPhi$.
\end{proposition*}

\begin{proof}
Since $HD(\mathcal{F},\mathcal{F}')\le\varepsilon$, for every $F\in\mathcal{F}$ a $F'\in\mathcal{F}'$ and an $\eta>0$ exist such that
$D_{\GENEOH}(F,F')\le \varepsilon+\eta$.
The definition of $D_{\GENEOH}$ implies that $d_H(F(\varphi),F'(\varphi))\le\varepsilon+\eta$ for every $\varphi\in \varPhi$. From Theorem~\ref{t12} it follows that
\begin{equation}
d_{\mathrm{match}}(r_k(F(\varphi_1)),r_k(F'(\varphi_1)) \le \varepsilon+\eta
\end{equation}
and
\begin{equation}
d_{\mathrm{match}}(r_k(F(\varphi_2)),r_k(F'(\varphi_2))\le \varepsilon+\eta
\end{equation}
for every $\varphi_1,\varphi_2 \in \varPhi$.

Therefore,
\begin{equation}\left|d_{\mathrm{match}}(r_k(F(\varphi_1)),r_k(F(\varphi_2))-d_{\mathrm{match}}(r_k(F'(\varphi_1)),r_k(F'(\varphi_2))\right| \le 2(\varepsilon+\eta).
\end{equation}

As a consequence, $\mathcal{D}^{\mathcal{F},k}_{\mathrm{match}}(\varphi_1,\varphi_2)\le \mathcal{D}^{\mathcal{F}',k}_{\mathrm{match}}(\varphi_1,\varphi_2)+ 2(\varepsilon+\eta)$. We can show analogously that
$\mathcal{D}^{\mathcal{F}',k}_{\mathrm{match}}(\varphi_1,\varphi_2)\le \mathcal{D}^{\mathcal{F},k}_{\mathrm{match}}(\varphi_1,\varphi_2)+ 2(\varepsilon+\eta)$.
Since $\eta$ can be chosen arbitrarily small, from the previous two inequalities the proof of our statement follows.
\end{proof}


\begin{proposition*}[\ref{corapprox}]
Let $\mathcal{F}$ be a non-empty subset of $\mathcal{F}^{\mathrm{all}}$. For every $\varepsilon>0$, a finite subset $\mathcal{F}^{*}$ of $\mathcal{F}$ exists, such that
\begin{equation}
|\mathcal{D}^{\mathcal{F^{*}},k}_{\mathrm{match}}(\varphi_1, \varphi_2) - \mathcal{D}^{\mathcal{F},k}_{\mathrm{match}}(\varphi_1, \varphi_2)|\le \varepsilon
\end{equation}
for every $\varphi_1,\varphi_2 \in \varPhi$.
\end{proposition*}

\begin{proof}
Let us consider the closure $\bar{\mathcal{F}}$ of $\mathcal{F}$ in $\mathcal{F}^{\mathrm{all}}$. Let us also consider the covering $\mathcal{U}$ of $\bar{\mathcal{F}}$ obtained by taking all the open balls of radius $\frac{\varepsilon}{2}$ centered at points of $\mathcal{F}$, with respect to $D_{\GENEO}$. Theorem~\ref{t17} guarantees that $\mathcal{F}^{\mathrm{all}}$ is compact, hence also $\bar{\mathcal{F}}$ is compact. Therefore we can extract a finite covering $\{B_1, \dots, B_m \}$ of $\bar{\mathcal{F}}$ from $\mathcal{U}$. We can set $\mathcal{F}^{*}$ equal to the set of centers of the balls $B_1, \dots, B_m$. The statement of our corollary immediately follows from Proposition~\ref{14}, by recalling that $D_{\GENEOH}\le D_{\GENEO}$ and hence
$HD\left(\bar{\mathcal{F}},\mathcal{F}^*\right)\le\varepsilon/2$.
\end{proof}

\end{appendices}

\bibliographystyle{elsarticle-num}
\section*{\refname}
\bibliography{bibDNP}

\begin{thebibliography}{10}
\expandafter\ifx\csname url\endcsname\relax
  \def\url#1{\texttt{#1}}\fi
\expandafter\ifx\csname urlprefix\endcsname\relax\def\urlprefix{URL }\fi
\expandafter\ifx\csname href\endcsname\relax
  \def\href#1#2{#2} \def\path#1{#1}\fi

\bibitem{lecun1995convolutional}
Y.~LeCun, Y.~Bengio, et~al., {Convolutional networks for images, speech, and
  time series}, The handbook of brain theory and neural networks 3361~(10)
  (1995) 1995.

\bibitem{AnRoPo16}
F.~Anselmi, L.~Rosasco, T.~Poggio,
  \href{http://dx.doi.org/10.1093/imaiai/iaw009}{{On invariance and selectivity
  in representation learning}}, Information and Inference: A Journal of the IMA
  5~(2) (2016) 134--158.
\newblock \href
  {http://arxiv.org/abs//oup/backfile/content_public/journal/imaiai/5/2/10.1093_imaiai_iaw009/2/iaw009.pdf}
  {\path{arXiv:/oup/backfile/content_public/journal/imaiai/5/2/10.1093_imaiai_iaw009/2/iaw009.pdf}},
  \href {http://dx.doi.org/10.1093/imaiai/iaw009}
  {\path{doi:10.1093/imaiai/iaw009}}.
\newline\urlprefix\url{http://dx.doi.org/10.1093/imaiai/iaw009}

\bibitem{FrJa16}
P.~Frosini, G.~Jab{\l}o{\'n}ski,
  \href{http://dx.doi.org/10.1007/s00454-016-9761-y}{{Combining persistent
  homology and invariance groups for shape comparison}}, Discrete Comput. Geom.
  55~(2) (2016) 373--409.
\newblock \href {http://dx.doi.org/10.1007/s00454-016-9761-y}
  {\path{doi:10.1007/s00454-016-9761-y}}.
\newline\urlprefix\url{http://dx.doi.org/10.1007/s00454-016-9761-y}

\bibitem{cohen2016group}
T.~Cohen, M.~Welling, {Group equivariant convolutional networks}, in:
  International conference on machine learning, 2016, pp. 2990--2999.

\bibitem{worrall2017harmonic}
D.~E. Worrall, S.~J. Garbin, D.~Turmukhambetov, G.~J. Brostow, {Harmonic
  networks: Deep translation and rotation equivariance}, in: Proc. IEEE Conf.
  on Computer Vision and Pattern Recognition (CVPR), Vol.~2, 2017.

\bibitem{AdEmal17}
H.~Adams, T.~Emerson, M.~Kirby, R.~Neville, C.~Peterson, P.~Shipman,
  S.~Chepushtanova, E.~Hanson, F.~Motta, L.~Ziegelmeier,
  \href{http://dl.acm.org/citation.cfm?id=3122009.3122017}{Persistence images:
  A stable vector representation of persistent homology}, J. Mach. Learn. Res.
  18~(1) (2017) 218--252.
\newline\urlprefix\url{http://dl.acm.org/citation.cfm?id=3122009.3122017}

\bibitem{PuXiXi18}
C.~S. {Pun}, K.~{Xia}, S.~{Xian Lee}, {Persistent-homology-based machine
  learning and its applications -- A survey}, arXiv e-prints (2018)
  arXiv:1811.00252\href {http://arxiv.org/abs/1811.00252}
  {\path{arXiv:1811.00252}}.

\bibitem{GaCa18}
R.~B. Gabrielsson, G.~Carlsson,
  \href{http://arxiv.org/abs/1810.03234}{Exposition and interpretation of the
  topology of neural networks}, CoRR abs/1810.03234.
\newblock \href {http://arxiv.org/abs/1810.03234} {\path{arXiv:1810.03234}}.
\newline\urlprefix\url{http://arxiv.org/abs/1810.03234}

\bibitem{Fr16}
P.~Frosini, {Towards an Observer-oriented Theory of Shape Comparison}, in:
  A.~Ferreira, A.~Giachetti, D.~Giorgi (Eds.), Eurographics Workshop on 3D
  Object Retrieval, The Eurographics Association, 2016.
\newblock \href {http://dx.doi.org/10.2312/3dor.20161080}
  {\path{doi:10.2312/3dor.20161080}}.

\bibitem{Ca2009}
G.~Carlsson, \href{https://doi.org/10.1090/S0273-0979-09-01249-X}{{Topology and
  data}}, Bull. Amer. Math. Soc. (N.S.) 46~(2) (2009) 255--308.
\newblock \href {http://dx.doi.org/10.1090/S0273-0979-09-01249-X}
  {\path{doi:10.1090/S0273-0979-09-01249-X}}.
\newline\urlprefix\url{https://doi.org/10.1090/S0273-0979-09-01249-X}

\bibitem{Ca2013}
P.~Y. Lum, G.~Singh, A.~Lehman, T.~Ishkanov, M.~Vejdemo-Johansson,
  M.~Alagappan, J.~Carlsson, G.~E. Carlsson, {Extracting insights from the
  shape of complex data using topology}, in: Scientific reports, Vol.~3, 2013.

\bibitem{hatcher2005algebraic}
A.~Hatcher, Algebraic topology, 清华大学出版社有限公司, 2005.

\bibitem{BiDFFaal08}
S.~Biasotti, L.~De~Floriani, B.~Falcidieno, P.~Frosini, D.~Giorgi, C.~Landi,
  L.~Papaleo, M.~Spagnuolo,
  \href{http://doi.acm.org/10.1145/1391729.1391731}{{Describing Shapes by
  Geometrical-topological Properties of Real Functions}}, ACM Comput. Surv.
  40~(4) (2008) 12:1--12:87.
\newblock \href {http://dx.doi.org/10.1145/1391729.1391731}
  {\path{doi:10.1145/1391729.1391731}}.
\newline\urlprefix\url{http://doi.acm.org/10.1145/1391729.1391731}

\bibitem{CaZo09}
G.~Carlsson, A.~Zomorodian,
  \href{http://dx.doi.org/10.1007/s00454-009-9176-0}{{The theory of
  multidimensional persistence}}, Discrete Comput. Geom. 42~(1) (2009) 71--93.
\newblock \href {http://dx.doi.org/10.1007/s00454-009-9176-0}
  {\path{doi:10.1007/s00454-009-9176-0}}.
\newline\urlprefix\url{http://dx.doi.org/10.1007/s00454-009-9176-0}

\bibitem{EdHa08}
H.~Edelsbrunner, J.~Harer,
  \href{http://dx.doi.org/10.1090/conm/453/08802}{{Persistent homology---a
  survey}}, in: Surveys on discrete and computational geometry, Vol. 453 of
  Contemp. Math., Amer. Math. Soc., Providence, RI, 2008, pp. 257--282.
\newblock \href {http://dx.doi.org/10.1090/conm/453/08802}
  {\path{doi:10.1090/conm/453/08802}}.
\newline\urlprefix\url{http://dx.doi.org/10.1090/conm/453/08802}

\bibitem{CSEdHa07}
D.~Cohen-Steiner, H.~Edelsbrunner, J.~Harer,
  \href{http://dx.doi.org/10.1007/s00454-006-1276-5}{{Stability of persistence
  diagrams}}, Discrete Comput. Geom. 37~(1) (2007) 103--120.
\newblock \href {http://dx.doi.org/10.1007/s00454-006-1276-5}
  {\path{doi:10.1007/s00454-006-1276-5}}.
\newline\urlprefix\url{http://dx.doi.org/10.1007/s00454-006-1276-5}

\bibitem{CeDFFeal13}
A.~Cerri, B.~Di~Fabio, M.~Ferri, P.~Frosini, C.~Landi,
  \href{http://dx.doi.org/10.1002/mma.2704}{{Betti numbers in multidimensional
  persistent homology are stable functions}}, Math. Methods Appl. Sci. 36~(12)
  (2013) 1543--1557.
\newblock \href {http://dx.doi.org/10.1002/mma.2704}
  {\path{doi:10.1002/mma.2704}}.
\newline\urlprefix\url{http://dx.doi.org/10.1002/mma.2704}

\bibitem{Ga64}
S.~A. Gaal, {Point set topology}, Pure and Applied Mathematics, Vol. XVI,
  Academic Press, New York-London, 1964.

\bibitem{Ou15}
S.~Y. Oudot, \href{https://doi.org/10.1090/surv/209}{{Persistence theory: from
  quiver representations to data analysis}}, Vol. 209 of Mathematical Surveys
  and Monographs, American Mathematical Society, Providence, RI, 2015.
\newblock \href {http://dx.doi.org/10.1090/surv/209}
  {\path{doi:10.1090/surv/209}}.
\newline\urlprefix\url{https://doi.org/10.1090/surv/209}

\bibitem{zomorodian2010fast}
A.~Zomorodian, {Fast construction of the Vietoris-Rips complex}, Computers \&
  Graphics 34~(3) (2010) 263--271.

\bibitem{fabbri20082d}
R.~Fabbri, L.~D.~F. Costa, J.~C. Torelli, O.~M. Bruno, {2D Euclidean distance
  transform algorithms: A comparative survey}, ACM Computing Surveys (CSUR)
  40~(1) (2008) 2.

\bibitem{langfelder2007defining}
P.~Langfelder, B.~Zhang, S.~Horvath, {Defining clusters from a hierarchical
  cluster tree: the Dynamic Tree Cut package for R}, Bioinformatics 24~(5)
  (2007) 719--720.

\bibitem{nair2010rectified}
V.~Nair, G.~E. Hinton, {Rectified linear units improve restricted boltzmann
  machines}, in: Proceedings of the 27th international conference on machine
  learning (ICML-10), 2010, pp. 807--814.

\bibitem{glorot2010understanding}
X.~Glorot, Y.~Bengio, {Understanding the difficulty of training deep
  feedforward neural networks}, in: Proceedings of the thirteenth international
  conference on artificial intelligence and statistics, 2010, pp. 249--256.

\bibitem{Hi11}
G.~E. Hinton, A.~Krizhevsky, S.~D. Wang, {Transforming auto-encoders}, in:
  International Conference on Artificial Neural Networks, Springer, 2011, pp.
  44--51.

\bibitem{Hi17}
S.~Sabour, N.~Frosst, G.~E. Hinton,
  \href{http://papers.nips.cc/paper/6975-dynamic-routing-between-capsules.pdf}{{Dynamic
  Routing Between Capsules}}, in: I.~Guyon, U.~V. Luxburg, S.~Bengio,
  H.~Wallach, R.~Fergus, S.~Vishwanathan, R.~Garnett (Eds.), Advances in Neural
  Information Processing Systems 30, Curran Associates, Inc., 2017, pp.
  3856--3866.
\newline\urlprefix\url{http://papers.nips.cc/paper/6975-dynamic-routing-between-capsules.pdf}

\end{thebibliography}


<<<<<<< HEAD
@ARTICLE{PuXiXi18,
       author = {{Pun}, Chi Seng and {Xia}, Kelin and {Xian Lee}, Si},
        title = "{Persistent-homology-based machine learning and its applications -- A survey}",
      journal = {arXiv e-prints},
     keywords = {Mathematics - Algebraic Topology},
         year = "2018",
        month = "Nov",
          eid = {arXiv:1811.00252},
        pages = {arXiv:1811.00252},
archivePrefix = {arXiv},
       eprint = {1811.00252},
 primaryClass = {math.AT},
       adsurl = {https://ui.adsabs.harvard.edu/\#abs/2018arXiv181100252P},
      adsnote = {Provided by the SAO/NASA Astrophysics Data System}
}

@article{GaCa18,
  author    = {Rickard Br{\"{u}}el Gabrielsson and
               Gunnar Carlsson},
  title     = {Exposition and interpretation of the topology of neural networks},
  journal   = {CoRR},
  volume    = {abs/1810.03234},
  year      = {2018},
  url       = {http://arxiv.org/abs/1810.03234},
  archivePrefix = {arXiv},
  eprint    = {1810.03234},
  timestamp = {Tue, 30 Oct 2018 10:49:09 +0100},
  biburl    = {https://dblp.org/rec/bib/journals/corr/abs-1810-03234},
  bibsource = {dblp computer science bibliography, https://dblp.org}
}

@article{AdEmal17,
 author = {Adams, Henry and Emerson, Tegan and Kirby, Michael and Neville, Rachel and Peterson, Chris and Shipman, Patrick and Chepushtanova, Sofya and Hanson, Eric and Motta, Francis and Ziegelmeier, Lori},
 title = {Persistence images: A stable vector representation of persistent homology},
 journal = {J. Mach. Learn. Res.},
 issue_date = {January 2017},
 volume = {18},
 number = {1},
 month = jan,
 year = {2017},
 issn = {1532-4435},
 pages = {218--252},
 numpages = {35},
 url = {http://dl.acm.org/citation.cfm?id=3122009.3122017},
 acmid = {3122017},
 publisher = {JMLR.org},
 keywords = {dynamical systems, machine learning, persistence images, persistent homology, topological data analysis},
=======
@book{hatcher2005algebraic,
  title={Algebraic topology},
  author={Hatcher, Allen},
  year={2005},
  publisher={清华大学出版社有限公司}
>>>>>>> 03de6da7a3745baca8b5bd74f33b4ded7e9f8db8
}

@inproceedings{nair2010rectified,
  title={{Rectified linear units improve restricted boltzmann machines}},
  author={Nair, Vinod and Hinton, Geoffrey E},
  booktitle={Proceedings of the 27th international conference on machine learning (ICML-10)},
  pages={807--814},
  year={2010}
}

@article{langfelder2007defining,
  title={{Defining clusters from a hierarchical cluster tree: the Dynamic Tree Cut package for R}},
  author={Langfelder, Peter and Zhang, Bin and Horvath, Steve},
  journal={Bioinformatics},
  volume={24},
  number={5},
  pages={719--720},
  year={2007},
  publisher={Oxford University Press}
}

@article{fabbri20082d,
  title={{2D Euclidean distance transform algorithms: A comparative survey}},
  author={Fabbri, Ricardo and Costa, Luciano Da F and Torelli, Julio C and Bruno, Odemir M},
  journal={ACM Computing Surveys (CSUR)},
  volume={40},
  number={1},
  pages={2},
  year={2008},
  publisher={ACM}
}

@article{zomorodian2010fast,
  title={{Fast construction of the Vietoris-Rips complex}},
  author={Zomorodian, Afra},
  journal={Computers \& Graphics},
  volume={34},
  number={3},
  pages={263--271},
  year={2010},
  publisher={Elsevier}
}

@inproceedings{jaderberg2015spatial,
  title={{Spatial transformer networks}},
  author={Jaderberg, Max and Simonyan, Karen and Zisserman, Andrew and others},
  booktitle={Advances in neural information processing systems},
  pages={2017--2025},
  year={2015}
}

@inproceedings{glorot2010understanding,
  title={{Understanding the difficulty of training deep feedforward neural networks}},
  author={Glorot, Xavier and Bengio, Yoshua},
  booktitle={Proceedings of the thirteenth international conference on artificial intelligence and statistics},
  pages={249--256},
  year={2010}
}


@inproceedings{cohen2016group,
  title={{Group equivariant convolutional networks}},
  author={Cohen, Taco and Welling, Max},
  booktitle={International conference on machine learning},
  pages={2990--2999},
  year={2016}
}

@inproceedings{worrall2017harmonic,
  title={{Harmonic networks: Deep translation and rotation equivariance}},
  author={Worrall, Daniel E and Garbin, Stephan J and Turmukhambetov, Daniyar and Brostow, Gabriel J},
  booktitle={Proc. IEEE Conf. on Computer Vision and Pattern Recognition (CVPR)},
  volume={2},
  year={2017}
}

@article{lecun1995convolutional,
  title={{Convolutional networks for images, speech, and time series}},
  author={LeCun, Yann and Bengio, Yoshua and others},
  journal={The handbook of brain theory and neural networks},
  volume={3361},
  number={10},
  pages={1995},
  year={1995}
}

@incollection {EdMo13,
    AUTHOR = {Edelsbrunner, Herbert and Morozov, Dmitriy},
     TITLE = {{Persistent homology: theory and practice}},
 BOOKTITLE = {European {C}ongress of {M}athematics},
     PAGES = {31--50},
 PUBLISHER = {Eur. Math. Soc., Z\"urich},
      YEAR = {2013},
   MRCLASS = {55N35 (68U05 68W30)},
  MRNUMBER = {3469114},
}

@incollection {EdHa08,
    AUTHOR = {Edelsbrunner, Herbert and Harer, John},
     TITLE = {{Persistent homology---a survey}},
 BOOKTITLE = {Surveys on discrete and computational geometry},
    SERIES = {Contemp. Math.},
    VOLUME = {453},
     PAGES = {257--282},
 PUBLISHER = {Amer. Math. Soc., Providence, RI},
      YEAR = {2008},
   MRCLASS = {55N35 (52-02 55T05 57M99 68W01)},
  MRNUMBER = {2405684},
MRREVIEWER = {Greg Friedman},
       DOI = {10.1090/conm/453/08802},
       URL = {http://dx.doi.org/10.1090/conm/453/08802},
}

@inproceedings {Fr16,
booktitle = {Eurographics Workshop on 3D Object Retrieval},
editor = {A. Ferreira and A. Giachetti and D. Giorgi},
title = {{Towards an Observer-oriented Theory of Shape Comparison}},
author = {Frosini, Patrizio},
year = {2016},
publisher = {The Eurographics Association},
ISSN = {1997-0471},
ISBN = {978-3-03868-004-8},
DOI = {10.2312/3dor.20161080}
}

@book {Ha02,
    AUTHOR = {Hatcher, Allen},
     TITLE = {{Algebraic topology}},
 PUBLISHER = {Cambridge University Press, Cambridge},
      YEAR = {2002},
     PAGES = {xii+544},
      ISBN = {0-521-79160-X; 0-521-79540-0},
   MRCLASS = {55-01 (55-00)},
  MRNUMBER = {1867354},
MRREVIEWER = {Donald W. Kahn},
}

@article {CaZo09,
    AUTHOR = {Carlsson, Gunnar and Zomorodian, Afra},
     TITLE = {{The theory of multidimensional persistence}},
   JOURNAL = {Discrete Comput. Geom.},
  FJOURNAL = {Discrete \& Computational Geometry. An International Journal
              of Mathematics and Computer Science},
    VOLUME = {42},
      YEAR = {2009},
    NUMBER = {1},
     PAGES = {71--93},
      ISSN = {0179-5376},
     CODEN = {DCGEER},
   MRCLASS = {52C35 (68U05)},
  MRNUMBER = {2506738},
       DOI = {10.1007/s00454-009-9176-0},
       URL = {http://dx.doi.org/10.1007/s00454-009-9176-0},
}

@article {CSEdHa07,
    AUTHOR = {Cohen-Steiner, David and Edelsbrunner, Herbert and Harer,
              John},
     TITLE = {{Stability of persistence diagrams}},
   JOURNAL = {Discrete Comput. Geom.},
  FJOURNAL = {Discrete \& Computational Geometry. An International Journal
              of Mathematics and Computer Science},
    VOLUME = {37},
      YEAR = {2007},
    NUMBER = {1},
     PAGES = {103--120},
      ISSN = {0179-5376},
     CODEN = {DCGEER},
   MRCLASS = {68U05 (55N05)},
  MRNUMBER = {2279866},
       DOI = {10.1007/s00454-006-1276-5},
       URL = {http://dx.doi.org/10.1007/s00454-006-1276-5},
}

@article {FrJa16,
    AUTHOR = {Frosini, Patrizio and Jab{\l}o{\'n}ski, Grzegorz},
     TITLE = {{Combining persistent homology and invariance groups for shape
              comparison}},
   JOURNAL = {Discrete Comput. Geom.},
  FJOURNAL = {Discrete \& Computational Geometry. An International Journal
              of Mathematics and Computer Science},
    VOLUME = {55},
      YEAR = {2016},
    NUMBER = {2},
     PAGES = {373--409},
      ISSN = {0179-5376},
   MRCLASS = {55N35 (47H09 54H15 57S10 65D18 68U05)},
  MRNUMBER = {3458602},
       DOI = {10.1007/s00454-016-9761-y},
       URL = {http://dx.doi.org/10.1007/s00454-016-9761-y},
}

@article{BiDFFaal08,
 author = {Biasotti, S. and De Floriani, L. and Falcidieno, B. and Frosini, P. and Giorgi, D. and Landi, C. and Papaleo, L. and Spagnuolo, M.},
 title = {{Describing Shapes by Geometrical-topological Properties of Real Functions}},
 journal = {ACM Comput. Surv.},
 issue_date = {October 2008},
 volume = {40},
 number = {4},
 month = oct,
 year = {2008},
 issn = {0360-0300},
 pages = {12:1--12:87},
 articleno = {12},
 numpages = {87},
 url = {http://doi.acm.org/10.1145/1391729.1391731},
 doi = {10.1145/1391729.1391731},
 acmid = {1391731},
 publisher = {ACM},
 address = {New York, NY, USA},
 keywords = {Computational topology, Morse complexes, Morse theory, Reeb graph, contour tree, persistent homology, shape analysis, size theory},
}

@book {Ga64,
    AUTHOR = {Gaal, Steven A.},
     TITLE = {{Point set topology}},
    SERIES = {Pure and Applied Mathematics, Vol. XVI},
 PUBLISHER = {Academic Press, New York-London},
      YEAR = {1964},
     PAGES = {xi+317},
   MRCLASS = {54.00},
  MRNUMBER = {0171253},
MRREVIEWER = {M. L. Curtis},
}

@book {Su09,
    AUTHOR = {Sutherland, Wilson A.},
     TITLE = {{Introduction to metric and topological spaces}},
 PUBLISHER = {Oxford University Press, Oxford},
      YEAR = {2009},
}

@book {Wi70,
    AUTHOR = {Willard, Stephen},
     TITLE = {General Topology},
 PUBLISHER = {Courier Corporation},
      YEAR = {1970},
}

@article {CeDFFeal13,
    AUTHOR = {Cerri, Andrea and Di Fabio, Barbara and Ferri, Massimo and
              Frosini, Patrizio and Landi, Claudia},
     TITLE = {{Betti numbers in multidimensional persistent homology are
              stable functions}},
   JOURNAL = {Math. Methods Appl. Sci.},
  FJOURNAL = {Mathematical Methods in the Applied Sciences},
    VOLUME = {36},
      YEAR = {2013},
    NUMBER = {12},
     PAGES = {1543--1557},
      ISSN = {0170-4214},
   MRCLASS = {55N35},
  MRNUMBER = {3083259},
MRREVIEWER = {Clara L{\"o}h},
       DOI = {10.1002/mma.2704},
       URL = {http://dx.doi.org/10.1002/mma.2704},
}

@book {Ou15,
    AUTHOR = {Oudot, Steve Y.},
     TITLE = {{Persistence theory: from quiver representations to data
              analysis}},
    SERIES = {Mathematical Surveys and Monographs},
    VOLUME = {209},
 PUBLISHER = {American Mathematical Society, Providence, RI},
      YEAR = {2015},
     PAGES = {viii+218},
      ISBN = {978-1-4704-2545-6},
   MRCLASS = {55N35 (16G20 55U10 62-07 68U05)},
  MRNUMBER = {3408277},
MRREVIEWER = {Patrizio Frosini},
       DOI = {10.1090/surv/209},
       URL = {https://doi.org/10.1090/surv/209},
}

@article{AnRoPo16,
author = {Anselmi, Fabio and Rosasco, Lorenzo and Poggio, Tomaso},
title = {{On invariance and selectivity in representation learning}},
journal = {Information and Inference: A Journal of the IMA},
volume = {5},
number = {2},
pages = {134-158},
year = {2016},
doi = {10.1093/imaiai/iaw009},
URL = {http://dx.doi.org/10.1093/imaiai/iaw009},
eprint = {/oup/backfile/content_public/journal/imaiai/5/2/10.1093_imaiai_iaw009/2/iaw009.pdf}
}

@inproceedings{Hi11,
  title={{Transforming auto-encoders}},
  author={Hinton, Geoffrey E and Krizhevsky, Alex and Wang, Sida D},
  booktitle={International Conference on Artificial Neural Networks},
  pages={44--51},
  year={2011},
  organization={Springer}
}

@incollection{Hi17,
title = {{Dynamic Routing Between Capsules}},
author = {Sabour, Sara and Frosst, Nicholas and Hinton, Geoffrey E},
booktitle = {Advances in Neural Information Processing Systems 30},
editor = {I. Guyon and U. V. Luxburg and S. Bengio and H. Wallach and R. Fergus and S. Vishwanathan and R. Garnett},
pages = {3856--3866},
year = {2017},
publisher = {Curran Associates, Inc.},
url = {http://papers.nips.cc/paper/6975-dynamic-routing-between-capsules.pdf}
}

@inproceedings{Ca2013,
  title={{Extracting insights from the shape of complex data using topology}},
  author={Pek Y. Lum and Gurvinder Singh and Amy Lehman and T. Ishkanov and Mikael Vejdemo-Johansson and Muthuraman Alagappan and Jan Carlsson and Gunnar E. Carlsson},
  booktitle={Scientific reports},
  volume = {3},
  year={2013}
}

@article {Ca2009,
    AUTHOR = {Carlsson, Gunnar},
     TITLE = {{Topology and data}},
   JOURNAL = {Bull. Amer. Math. Soc. (N.S.)},
  FJOURNAL = {American Mathematical Society. Bulletin. New Series},
    VOLUME = {46},
      YEAR = {2009},
    NUMBER = {2},
     PAGES = {255--308},
      ISSN = {0273-0979},
   MRCLASS = {55-02 (62-02 62-07 94A12)},
  MRNUMBER = {2476414},
MRREVIEWER = {Peter Bubenik},
       DOI = {10.1090/S0273-0979-09-01249-X},
       URL = {https://doi.org/10.1090/S0273-0979-09-01249-X},
}

\end{document}